\documentclass{article}

\usepackage[preprint]{neurips_2026}

\usepackage[utf8]{inputenc} 
\usepackage[T1]{fontenc}    
\usepackage{hyperref}       
\usepackage{url}            
\usepackage{booktabs}       
\usepackage{amsfonts}       
\usepackage{nicefrac}       
\usepackage{amsmath,amssymb,amsthm,mathtools,bm}
\usepackage{microtype}      
\usepackage{lipsum}
\usepackage{graphicx}
\usepackage{wrapfig}
\usepackage{cleveref}
\graphicspath{ {./images/} }

\usepackage[compact]{titlesec} 
\titlespacing{\section}{0pt}{0.5ex}{0ex} 
\titlespacing{\subsection}{0pt}{0.5ex}{0ex} 
\titlespacing{\subsubsection}{0pt}{0.5ex}{0ex}

\usepackage{thmtools}
\usepackage[framemethod=TikZ]{mdframed}

\mdfdefinestyle{thmstyle}{
  linecolor=blue!40!black,
  backgroundcolor=blue!3,
  linewidth=1.2pt,
  innertopmargin=6pt,
  innerbottommargin=6pt,
  innerleftmargin=8pt,
  innerrightmargin=8pt,
  skipabove=8pt,
  skipbelow=4pt,
  roundcorner=2pt,
}
\mdfdefinestyle{defstyle}{
  linecolor=teal!60!black,
  backgroundcolor=teal!4,
  linewidth=1pt,
  innertopmargin=6pt,innerbottommargin=6pt,
  innerleftmargin=8pt,innerrightmargin=8pt,
  skipabove=8pt,skipbelow=4pt,roundcorner=2pt,
}
\mdfdefinestyle{propstyle}{
  linecolor=orange!70!black,
  backgroundcolor=orange!4,
  linewidth=1pt,
  innertopmargin=6pt,innerbottommargin=6pt,
  innerleftmargin=8pt,innerrightmargin=8pt,
  skipabove=8pt,skipbelow=4pt,roundcorner=2pt,
}
\mdfdefinestyle{corstyle}{
  linecolor=purple!60!black,
  backgroundcolor=purple!4,
  linewidth=1pt,
  innertopmargin=6pt,innerbottommargin=6pt,
  innerleftmargin=8pt,innerrightmargin=8pt,
  skipabove=8pt,skipbelow=4pt,roundcorner=2pt,
}
\mdfdefinestyle{algostyle}{
  linecolor=gray!60,
  backgroundcolor=gray!6,
  linewidth=0.8pt,
  innertopmargin=6pt,innerbottommargin=6pt,
  innerleftmargin=8pt,innerrightmargin=8pt,
  skipabove=8pt,skipbelow=4pt,
}

\newtheorem{definition}{Definition}
\newtheorem{proposition}{Proposition}
\newtheorem{theorem}{Theorem}
\newtheorem{corollary}{Corollary}

\declaretheorem[style=remark,name=Constraint,numberwithin=section]{constraint}

\newcommand{\R}{\mathbb{R}}

\newcommand{\calK}{\mathcal{K}}
\newcommand{\calG}{\mathcal{G}}

\newcommand{\calV}{\mathcal{V}}
\newcommand{\calE}{\mathcal{E}}
\newcommand{\Thetaset}{\Theta}

\newcommand{\GP}{\mathcal{GP}}
\newcommand{\E}{\mathbb{E}}
\newcommand{\Prob}{\mathbb{P}}
\newcommand{\TV}{d_{\mathrm{TV}}}
\newcommand{\bbm}{\mathbf{1}}
\newcommand{\argmax}{\operatorname*{arg\,max}}
\newcommand{\argmin}{\operatorname*{arg\,min}}

\newcommand{\Hb}{H_{\mathrm{b}}}
\newcommand{\etabar}{\bar{\eta}}
\newcommand{\zetaT}{\zeta_T}
\newcommand{\regret}{\mathcal{R}}
\newcommand{\sigmaR}{\sigma_{\mathcal{R}}}

\usepackage{algorithm}
\usepackage{algpseudocode}
\algnewcommand{\LineComment}[1]{\State \(\triangleright\) \textit{#1}}

\usepackage{booktabs}
\usepackage{array}
\usepackage{multirow}
\usepackage{xcolor}
\usepackage{colortbl}

\usepackage{graphicx}
\usepackage{subcaption}

\usepackage{enumitem}
\setlist[itemize]{leftmargin=1.5em,itemsep=2pt,topsep=3pt}
\setlist[enumerate]{leftmargin=1.5em,itemsep=2pt,topsep=3pt}

\title{ADKO: Agentic Decentralized Knowledge Optimization\thanks{Code: \url{https://github.com/lucasrillo/adko}}}

\author{
Lucas Nerone Rillo$^{1}$ \\
\texttt{lucasnr@iastate.edu}
\And
Zhanhong Jiang$^{1}$ \\
\texttt{zhjiang@iastate.edu}
\And
Nastaran Saadati$^{1}$ \\
\texttt{nsaadati@iastate.edu}
\And
Aditya Balu$^{1}$ \\
\texttt{baditya@iastate.edu}
\And
Baskar Ganapathysubramanian$^{1}$ \\
\texttt{baskarg@iastate.edu}
\And
Chinmay Hegde$^{2}$ \\
\texttt{chinmay.h@nyu.edu}
\And
Soumik Sarkar$^{1}$ \\
\texttt{soumiks@iastate.edu}
\\[1em]
$^{1}$ Iowa State University
\qquad
$^{2}$ New York University
}

\begin{document}
\maketitle
\begin{abstract}

We present \textbf{Agentic Decentralized Knowledge Optimization (ADKO)}, a framework for collaborative black-box optimization across autonomous agents that achieves sample efficiency, privacy preservation, heterogeneous-objective handling, and communication efficiency. Each agent maintains a private Gaussian Process (GP) surrogate trained on local data and communicates only through \emph{knowledge tokens}—compact, lossy summaries containing directional signals, advantage scores, and optional language-model (LM) insights—without sharing raw data or model parameters. ADKO unifies GP-Upper Confidence Bound (GP-UCB), parallel Bayesian optimization, decentralized learning, and LM-guided discovery. We provide the first formal analysis of dual information loss: token compression, quantified via mutual-information–based fidelity, and LM approximation error, decomposed into bias and stochastic noise. Our main result shows cumulative regret decomposes into GP error, LM bias, LM noise, and compression loss, with necessary and sufficient conditions for sublinear regret. We also propose fidelity-aware token pruning to preserve high-information tokens under memory budget. Experiments on neural architecture search and scientific discovery validate the theory and show consistent improvements over strong baselines.
\end{abstract}


\begin{figure}[htbp]
    \centering
    \includegraphics[width=1\linewidth]{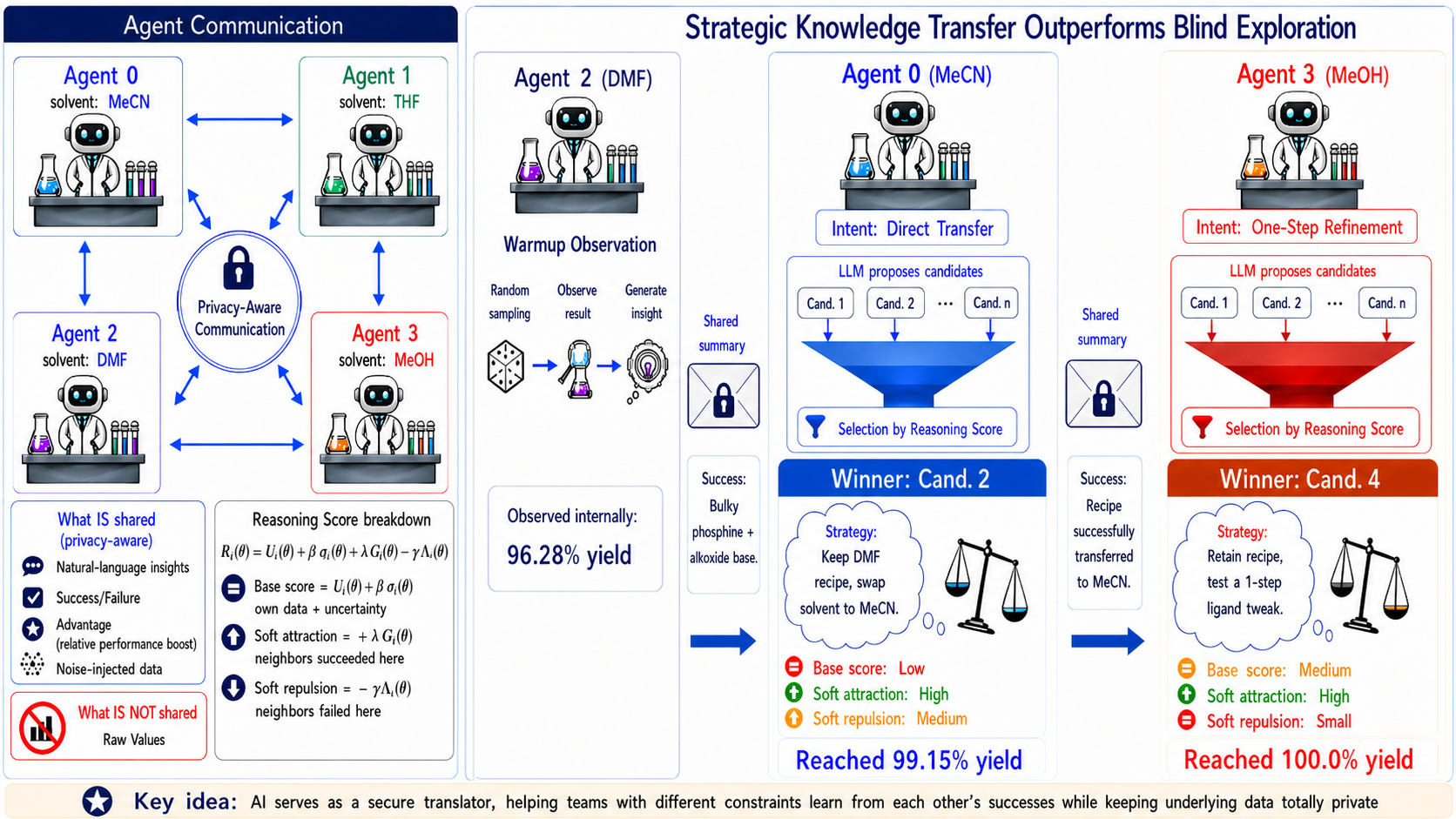}
    \caption{Illustrative example of decentralized knowledge transfer in ADKO for heterogeneous chemical optimization. Agents operating under different solvent constraints exchange only privacy-aware knowledge tokens rather than raw experimental data. The example shows how a high-yield reaction discovered by one agent is semantically transferred and refined by neighboring agents through LM-guided reasoning and token-based communication, enabling strategic collaboration that outperforms blind exploration while preserving data privacy.}
    \label{fig:lab_protocol}
\end{figure}

\section{Introduction}
Optimization of expensive black-box functions is a central challenge across science and engineering, traditionally studied in the single-agent setting via Bayesian Optimization (BO)~\cite{movckus1974bayesian,jones1998efficient,srinivas2009gaussian}.
However, many real-world problems are inherently distributed: multiple organizations conduct parallel experiments on shared design spaces~\cite{dow2011prototyping}, robotic systems explore partially overlapping environments~\cite{solanas2004coordinated}, hospitals optimize treatments under data-sovereignty constraints~\cite{hamiti2024data}, and ML teams search architectures while keeping datasets local~\cite{sparks2015automating}. In these settings, agents hold private data that cannot be shared, yet could benefit significantly from coordination without a central authority. This raises a fundamental question: \textit{what is the minimal information agents should exchange to approach the performance of a fully centralized counterpart?} We answer this with the concept of a \textbf{knowledge token}—a compact, privacy-preserving, and information-theoretically grounded summary of experimental outcomes—and the \textbf{Agentic Decentralized Knowledge Optimization (ADKO)} framework for reasoning over such tokens. Leveraging pre-trained language models (LMs) as prior knowledge bases, ADKO enables agents to share semantically rich insights (e.g., mechanistic explanations or empirical heuristics) that others can interpret and exploit, achieving effective collaboration without exposing raw data.

Despite significant recent progress on individual components, such as distributed
BO~\cite{gonzalez2016batch}, privacy-preserving BO~\cite{kusner2015differentially},
federated learning~\cite{mcmahan2017communication}, and language model (LM)-guided
optimization~\cite{yang2023large,liu2024large}, few existing frameworks \textit{simultaneously}
handle: (1) expensive black-box evaluations with per-agent budget constraints;
(2) hard privacy preventing any raw-data sharing; (3) heterogeneous private
objectives from different data distributions; and (4) natural-language semantic
communication via a language model.
More importantly, no existing theoretical analysis addresses the \emph{joint
effect} of token compression and LM approximation error on regret, which are the two
information losses that are both unavoidable and coupled in the ADKO setting.
The information-loss problem is twofold and inherently coupled. First, compressing observations into binary tokens (e.g., \textsc{SUCCESS}/\textsc{FAIL}) irreversibly discards magnitude information, with fidelity depending on distance from a pre-defined contextual baseline: near-boundary outcomes carry little information, while extreme outcomes are nearly lossless. Second, when an LM reasons over these tokens, it introduces approximation error by relying on a prior that may misalign with the true landscape. These losses compound: high-fidelity tokens are ineffective under a miscalibrated LM, and accurate LMs cannot compensate for low-fidelity inputs. Our theory explicitly characterizes both sources and their interaction.


\textbf{Summary of Contributions.}
We present ADKO, a modular framework for collaborative black-box optimization, as shown in Figure~\ref{fig:lab_protocol}. ADKO enables autonomous agents to optimize heterogeneous objectives through a private-surrogate architecture that replaces raw data sharing with privacy-preserving, high-fidelity knowledge tokens. Our work bridges the gap between decentralized Bayesian Optimization and Large Language Model (LM) reasoning, providing a rigorous theoretical foundation for how information loss at the edge impacts global convergence. Specifically the contributions can be summarized as:
\textbf{Framework.} We introduce ADKO, combining private Gaussian Process (GP) surrogates, graph-structured token communication, and an optional LM reasoning module into a unified loop.
The framework is explicitly modular: any component can be ablated, with clear
theoretical predictions about the resulting performance degradation. Table~\ref{tab:comparison} in Section~\ref{sec:methodological_comparison} summarizes the comparison. \textbf{Information-loss theory.} We provide the first formal treatment of token compression loss in
decentralized BO, introducing token fidelity $\eta_k \in [0,1]$ (Definition~\ref{def:fidelity}) and decomposing LM approximation error into systematic bias bounded by the total variation distance and stochastic noise that decays with token accumulation (Proposition~\ref{prop:lmerror}).
\textbf{Sublinear regret \& fidelity-aware pruning.} We prove a cumulative regret bound (Theorem~\ref{thm:main}) that additively separates four sources of regret: GP error, LM bias, LM noise, and token compression, showing sublinear regret
(Corollary~\ref{cor:sublinear}). In addition, we prove that fidelity-aware pruning maintains average fidelity approaching one under any fixed communication budget (Proposition~\ref{prop:pruning}). \textbf{Empirical validation.}
Experiments on neural architecture search and scientific discovery validate the theory and demonstrate practical advantages, while selectively disabling the LM module in some settings isolates the contribution of token-based collaboration apart from semantic reasoning.

\section{Related Work}
\label{sec:related_work}

\textbf{Bayesian Optimization.}
BO~\cite{movckus1974bayesian,brochu2010tutorial} utilizes Gaussian Process (GP) surrogates and acquisition functions like UCB~\cite{auer2002using,srinivas2009gaussian} achieving $O(\sqrt{T\zetaT})$ ($\zetaT$ is the
\emph{maximal information gain}) regret~\cite{srinivas2009gaussian,pmlr-v130-vakili21a}, yet standard extensions for multi-fidelity BO~\cite{kandasamy2017multi}, high-dimensional BO~\cite{wang2016bayesian,eriksson2019scalable}, and parallel and distributed BO~\cite{gonzalez2016batch,ginsbourger2010kriging,NIPS2012_05311655, azimi2012hybrid} largely ignore the privacy and communication constraints of decentralized systems. While asynchronous and consensus-based methods~\cite{desautels2014parallelizing,yue2025collaborative} offer partial decentralization, they lack mechanisms for communication-induced information loss. ADKO overcomes these limitations by replacing raw data sharing with token-based communication, explicitly modeling compression loss in regret analysis. By encoding outcomes into ordinal tokens relative to contextual baselines, ADKO provides a semantic privacy framework, which is distinct from noise-based differential privacy \cite{kusner2015differentially} and aligns with federated quantization~\cite{alistarh2017qsgd} through a robust, interpretable mutual-information lens.


\textbf{Distributed Learning.}
Federated learning (FL)~\cite{mcmahan2017communication,li2020federated} preserves data locality but struggles with non-IID distributions~\cite{karimireddy2020scaffold} a challenge analogous to the heterogeneous objectives addressed by ADKO. However, ADKO diverges from FL in three fundamental ways: the "model" is a non-shareable physical or computational process, agents pursue distinct objectives making gradient aggregation ineffective, and coordination is achieved via single binary tokens rather than high-frequency gradient exchange. While active FL~\cite{goetz2019active,10197242,10833800} focuses on label selection, ADKO’s acquisition function prioritizes experimental design in expensive black-box settings. Furthermore, unlike decentralized gradient methods~\cite{nedic2009distributed,jiang2017collaborative,10542323,10251949} that rely on sharing parameters or gradients for consensus, or multi-agent RL~\cite{lowe2017multi,NEURIPS2022_9c1535a0} which is too sample-intensive for costly evaluations, ADKO occupies a unique niche by providing provable convergence guarantees under strict privacy constraints without gradient access.


\textbf{Language Models as Optimizers and Reasoners.} 
Emerging research utilizes LMs as optimizers~\cite{yang2023large,chen2023evoprompting,romera2024mathematical} or surrogates~\cite{liu2024large,BRAHMACHARY2025129272,muller2021transformers} yet these methods often discard uncertainty quantification and lack theoretical regret analysis. ADKO offers a complementary approach: it retains Gaussian Processes for principled uncertainty on private data while leveraging LMs for domain-informed proposals and token encoding. Unlike LLM-guided scientific agents~\cite{boiko2023emergent,bran2023chemcrow,ren2026scientificintelligencesurveyllmbased,Li_Li_Wu_Liao_HAO_Shao_Xu_2026,hu2026oscagentacceleratingdiscoveryorganic,swanson2025virtual,jin2025gridmind} that lack convergence proofs, ADKO provides a multi-agent framework with formal guarantees.

\section{Problem Formulation and Preliminaries}
\label{sec:problem}
\textbf{General Setting.}
Consider $N$ autonomous agents indexed by $i \in [N] = \{1,\ldots,N\}$
optimizing over a shared compact design space $\Thetaset \subset \R^d$.
Agents are connected by an undirected communication graph
$\calG = (\calV, \calE)$ with $\calV = [N]$ and $\calE \subseteq \calV \times \calV$.
The neighborhood of agent $i$ is
$\mathcal{N}_i = \{j : (i,j) \in \calE\} \cup \{i\}$.
The graph $\calG$ is assumed to be connected throughout the rest of the paper; its algebraic connectivity
(Fiedler value) $\lambda_2(L(\calG)) > 0$ is the key topological quantity, where $L(\cdot)$ is the graph Laplacian. We next define the local objective and the relevant observations only known to a specific agent $i$.

\begin{definition}[Private Objectives and Observations]
\label{def:objectives}
Each agent $i$ has a private objective $f_i : \Thetaset \to \R$ drawn from a
shared Gaussian Process prior $\GP(0, g)$ with Mat\'ern-$5/2$ kernel $g$.
Agent $i$ observes noisy evaluations through all time $t$:
$
  y_i^t = f_i(\theta_i^t) + \varepsilon_i^t,
  \varepsilon_i^t \sim \mathcal{N}(0, \sigma^2),
  \text{i.i.d.\ over } t,
$
where $\sigma^2$ is variance. It maintains private dataset $D_i^t = \{(\theta_i^s, y_i^s)\}_{s=1}^t$.
No agent ever accesses $D_j^t$ for $j \neq i$.
\end{definition}

The objectives $\{f_i\}$ are (possibly) \emph{heterogeneous}: while drawn from the same
GP prior (sharing structural regularity), they need not be equal. With this in hand, we are now ready to formally define the global objective, which is more generalized than those in existing works~\cite{nedic2009distributed,jiang2017collaborative}.

\begin{definition}[Global Objective]
\label{def:global}
The global objective $F : \Thetaset \to \R$ aggregates individual objectives:
$
  F(\theta) = \Phi\bigl(f_1(\theta), \ldots, f_N(\theta)\bigr),
$
where $\Phi : \R^N \to \R$ is a known aggregation operator.
Standard choices include:
\textit{mean} $F(\theta) = \frac{1}{N}\sum_i f_i(\theta)$;
\textit{worst-case} $F(\theta) = \min_i f_i(\theta)$ (robustness);
\textit{weighted sum} $F(\theta) = \sum_i w_i f_i(\theta)$, $w_i>0$ and $\sum_iw_i=1$. 
The global optimum is $\theta^* = \argmax_{\theta \in \Thetaset} F(\theta)$.
\end{definition}

Many real-world optimization tasks such as distributed sensing, collaborative model tuning, and privacy-sensitive scientific discovery are inherently decentralized, with agents unable to share raw data and operating under strict communication limits. ADKO is designed for this setting: each agent maintains a private GP surrogate to protect local data, while graph-structured knowledge tokens enable efficient, bounded information exchange. This design ensures both privacy preservation and communication efficiency, motivating the formal constraints introduced next.
\begin{constraint}[Hard Privacy]
\label{con:privacy}
Agent $i$ may never transmit $D_i^t$, any element $(\theta^t_i, y_i^t)$, any
GP parameter inferred from $D_i^t$, or any sufficient statistic of $D_i^t$,
to any other agent. The only permitted transmission is a knowledge token
$k_i^t$ defined in Section~\ref{sec:algorithm}.
\end{constraint}

\begin{constraint}[Communication Budget]
\label{con:budget}
At each round $t$, agent $i$ may transmit at most one token $k_i^t$
of bounded size ($\leq B$ bits) to each neighbour
$j \in \mathcal{N}_i$. Token memory per agent is bounded:
$|\calK_i^t| \leq B$ at all times, where $\mathcal{K}^t_i$ is the token memory buffer specific to agent $i$ at $t$.
\end{constraint}

\textbf{Gaussian Process Background.}
Given prior $\GP(0,g)$ and $n$ observations
$D = \{(\theta_j, y_j)\}_{j=1}^n$, the posterior is Gaussian with:
\begin{align*}
  \mu_n(\theta) = g_n(\theta)^\top (\mathfrak{G}_n + \sigma^2 I)^{-1} y_n,
  \sigma_n^2(\theta) = g(\theta,\theta) - g_n(\theta)^\top
    (\mathfrak{G}_n + \sigma^2 I)^{-1} g_n(\theta),
\end{align*}
where $g_n(\theta) = [g(\theta_1,\theta),\ldots,g(\theta_n,\theta)]^\top$ are the kernel functions,
and $(\mathfrak{G}_n)_{ij} = g(\theta_i,\theta_j)$ with $\mathfrak{G}_n$ being the covariance matrix.
The \emph{maximal information gain} after $T$ evaluations is
$\zetaT = \max_{A \subseteq \Thetaset,\,|A|=T} I(f;\, y_A)$,
which satisfies $\zetaT = O\bigl(T^{d/(2\nu+d)}(\log T)^{d+1}\bigr)$
for Mat\'ern-$\nu$ kernels~\cite{srinivas2009gaussian}. $I(\cdot,\cdot)$ is the mutual information.
For Mat\'ern-$5/2$ ($\nu=5/2$), this is $o(T)$ for all $d \geq 1$.

\textbf{Regret Criterion.}
The per-round simple regret of agent $i$ at round $t$ is:
$
  r_i^t = F(\theta^*) - F(\theta_i^t),
  \theta_i^t = \argmax_{\theta \in \Theta} \tilde{R}_i(\theta),$ where $\tilde{R}_i$ is the degraded reasoning score defined in Section~\ref{sec:degraded_reasoning_score}.
The cumulative $N$-agent regret is
$\regret_N^T = \sum_{i=1}^N \sum_{t=1}^T r_i^t$.
Sublinear regret $\regret_N^T = o(T)$ implies convergence:
$\bar{r} = \regret_N^T/T \to 0$.

\section{The ADKO Algorithm}\label{sec:algorithm}
In this section, we detail the ADKO algorithm with definitions of knowledge token and reasoning score. We also discuss how ADKO enables diverse agents to collaborate and learn from each other.

\textbf{Knowledge Token.}
After executing experiment $\theta_i^t$ and observing $y_i^t$,
agent $i$'s LM encodes the outcome into a structured token:
$
  k_i^t = \{s_i^t,\; c_i^t,\; z_i^t,\;\varphi(\theta_i^t)\},
$
where:
$s_i^t \in \{\textsc{success},\textsc{fail}\}$:
\textit{directional signal}, $s_i^t = \textsc{success}$ iff $y_i^t \geq b^i_t$, where $b^i_t>0$ is a contextual baseline.
This binary quantization is the primary source of compression loss.
$c_i^t \in [0,1]$: \textit{advantage score}, $c_i^t = |y_i^t - b^i_t| / \|y - b^i_t\|_{\max}$.
It acts as an evidence of strength, measuring how decisively an observation supports a success or failure signal relative to a contextual baseline.
High advantage score means the observation is far from the baseline $b^i_t$---a nearly lossless encoding.
$z_i^t$: optional LM-generated natural-language insight (e.g., ``high temperature destabilizes this material class'').
$\varphi(\theta_i^t)$: an embedding of
the design point with $\varphi(\cdot)$ an embedding function (e.g., differential privacy, quantization).

The knowledge token $k^t_i$ is the core abstraction enabling ADKO to achieve decentralized collaboration under strict privacy and communication constraints. By compressing each observation into a binary outcome with an advantage score and optional semantic insight, it acts as a minimal, privacy-preserving surrogate that prevents raw data reconstruction while still enabling effective coordination. The inclusion of language-based insights further elevates tokens from numeric summaries to semantic carriers, allowing agents to leverage domain knowledge without sharing models or data. Knowledge tokens turn decentralized optimization from parameters exchange into structured knowledge sharing—a minimal yet sufficient interface for provable, communication‑efficient collaboration.

\textbf{Reasoning Score.}
The reasoning score is the engine of each agent's autonomous
decision-making. Inspired by~\cite{srinivas2009gaussian,kennedy1995particle}, it synthesizes four distinct streams of evidence into a
single number that answers the question: \emph{``of all the candidate
points I could experiment on next, which one is worth the most?''} Formally, it is defined as follows:
\begin{equation}
  \label{eq:reasoning}
  R_i(\theta)
  = \underbrace{U_i(\theta)}_{\substack{\text{What I expect}\\\text{from my own data}}}
  + \underbrace{\beta\,\sigma_i(\theta)}_{\substack{\text{How uncertain}\\\text{I am personally}}}
  + \underbrace{\lambda\, G_i(\theta)}_{\substack{\text{Do my neighbours}\\\text{succeed here?}}}
  - \underbrace{\gamma\,\Lambda_i(\theta)}_{\substack{\text{Do my neighbours}\\\text{fail here?}}}
\end{equation}

Each term in the reasoning score plays a distinct epistemic role, combining private learning with social influence. The first two terms, $U_i(\theta)=\mu_i(\theta)$ and $\beta\,\sigma_i(\theta)$, correspond to standard GP-UCB~\cite{srinivas2009gaussian}: the posterior mean drives private exploitation based on local data, while the variance promotes private exploration in uncertain regions. The remaining terms introduce collaboration. The success term $\lambda G_i(\theta)$, where $G_i(\theta)
  = \sum_{j \in \mathcal{N}_i} \pi_{ij}
    \sum_{k \in \calK_j} c_k\, S(\theta, \theta_k)\,
    \bbm[s_k = \textsc{success}]$, aggregates high-advantage neighbor successes, \emph{softly attracting} the agent toward promising regions without sharing raw data. $S(\theta,\theta_k) =
\exp\!\bigl(-\|\varphi(\theta)-\varphi(\theta_k)\|^2/\sigma_s^2\bigr)$
is a similarity kernel and $\pi_{ij}$ are graph-normalized weights, determined by the connected graph itself. $\pi_{ij}\in\Pi\in\mathbb{R}^{N\times N}$ where $\Pi$ is the mixing matrix governing the convergence speed. Please see Section~\ref{sec:connection_fiedler_spectral_gap} for the connection between Fiedler value $\lambda_2(\mathcal{G})$ and the spectral gap $\rho(\Pi)$ associated with $\Pi$. Conversely, the failure term $\Lambda_i(\theta)$ mirrors $G_i$ penalizes regions where neighbors have failed, creating a shared "\emph{avoidance map}" that prevents redundant evaluations. $\Lambda_i(\theta)
  = \sum_{j \in \mathcal{N}_i} \pi_{ij}
    \sum_{k \in \calK_j} c_k\, S(\theta, \theta_k)\,
    \bbm[s_k = \textsc{fail}]$. Together, these terms enable agents to leverage collective experience while retaining autonomy and privacy. The weights $\beta, \lambda, \gamma$ control the balance between exploration, social attraction, and caution, while the kernel bandwidth $\sigma_s$ determines how far peer influence propagates in the design space. We refer interested readers to Section~\ref{analysis_reasoning} for more nuanced analysis and discussion on the reasoning score.
\begin{wrapfigure}{r}{0.5\linewidth}
    \centering
    \includegraphics[width=1\linewidth]{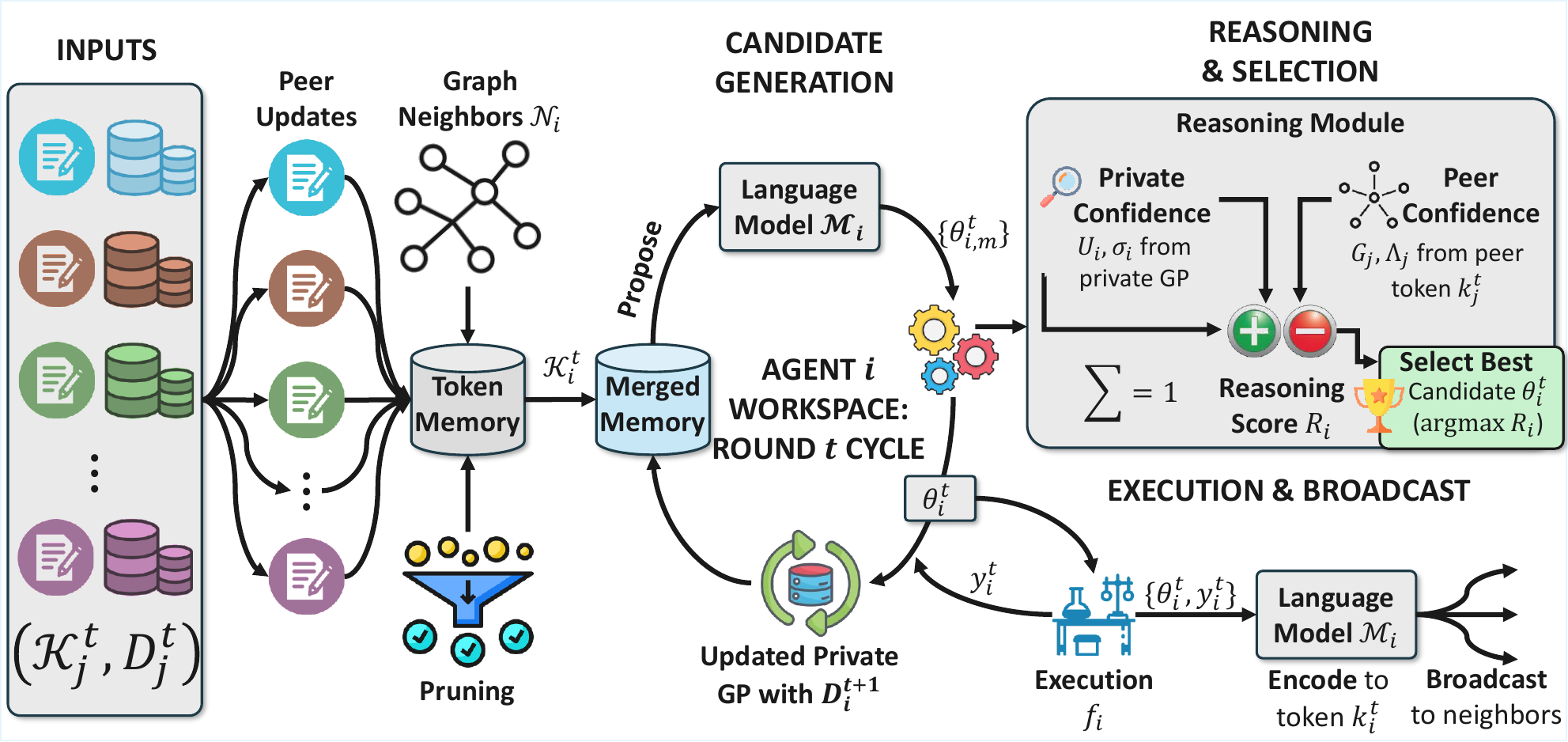}
    \caption{ADKO single agent round: Each agent iteratively aggregates neighbor tokens into a local buffer involving its own memory token, prunes to budget, proposes candidates via LM, and evaluates them using private advantage score and peer info. The best candidate is executed, with its outcome encoded into a compact privacy token and broadcast. The agent updates its private model, repeating the cycle for decentralized, communication-efficient collaboration without raw data sharing.}
    \label{fig:adko_diagram}
    \vspace{-7pt}
\end{wrapfigure}
\textbf{Full Algorithm.}
Algorithm~\ref{alg:adko} is not just an optimization routine but a protocol for decentralized collective intelligence under privacy and communication constraints, as shown in Figure~\ref{fig:adko_diagram}. Each agent encodes its experimental outcome $(\theta^t_i,y^t_i)$ into a knowledge token $k^t_i$ (Step 6)—a compact representation of success/failure ($s_i^t$), advantage score ($c^t_i$), and optional semantic insight—which is broadcast to neighbors and propagates across the network (Step 7). Through aggregation and fidelity-aware pruning (in Algorithm~\ref{alg:pruning} in Section~\ref{sec:token_pruning}), agents maintain a curated shared memory of high-quality information. This memory shapes decision-making by biasing candidate generation and modulating the acquisition function via success-attraction and failure-avoidance, enabling knowledge transfer without sharing raw data (Steps 2-4). Meanwhile, each agent updates its private GP with new observations, reinforcing useful signals or correcting misleading ones, ensuring robustness to heterogeneity (Step 8). From an individual perspective, the process remains \emph{fully autonomous}: the agent simply follows its own acquisition scores, subtly influenced by peer experience. Thus, collaboration emerges naturally from the token protocol, rather than explicit coordination, making ADKO a genuinely agentic and scalable approach to distributed optimization. Please check Section~\ref{sec:token_pruning} for more detail about token pruning.

\begin{algorithm}[htbp]
\small
\caption{ADKO --- Single Agent $i$, Round $t$}
\label{alg:adko}
\begin{algorithmic}[1]
  \Require Token memory $\calK_i^t$, GP dataset $D_i^t$,
           graph neighbors $\mathcal{N}_i$, language model $\mathcal{M}_i$
  \State \textbf{[Token Aggregation]}
    $\calK_i^t \leftarrow \textsc{merge}(\calK_i^{t-1},\;\{k_j^{t-1}\}_{j \in \mathcal{N}_i})$
  \State \hspace{1.2em} Apply fidelity-aware pruning (Algorithm~\ref{alg:pruning}) if $|\calK_i^t| > B$
  \State \textbf{[Candidate Generation via LM]}
    $\{\theta_{i,m}^t\}_{m=1}^M \leftarrow \mathcal{M}_i.\textsc{Propose}(\calK_i^t, \Thetaset)$
  \State \hspace{1.2em} Append local exploitation perturbations around current best
  \State \textbf{[Reasoning Evaluation]}
    For each $m$: compute
    $R_i(\theta_{i,m}^t)$ using \cref{eq:reasoning}
  \LineComment{$U_i$, $\sigma_i$ from GP posterior on $D_i^t$ (private)}
  \LineComment{$G_i$, $\Lambda_i$ from token memory $\calK_i^t$ (peer-sourced)}
  \State \textbf{[Self-Driven Selection]}
    $\theta_i^t \leftarrow \argmax_m R_i(\theta_{i,m}^t)$
  \State \textbf{[Execution]}
    $y_i^t \leftarrow f_i(\theta_i^t) + \varepsilon_i^t$
    \hfill $\triangleright$ physical/computational experiment
  \State \textbf{[Token Encoding via LM]}
    $k_i^t \leftarrow \mathcal{M}_i.\textsc{Encode}(\theta_i^t, y_i^t, \tau)$
  \State \textbf{[Broadcasting]}
    Send $k_i^t$ to all $j \in \mathcal{N}_i$
  \State \textbf{[Private GP Update]}
    $D_i^{t+1} \leftarrow D_i^t \cup \{(\theta_i^t, y_i^t)\}$; refit GP
  \Ensure Updated $\calK_i^{t+1}$, $D_i^{t+1}$; best $\theta_i^*$ found so far
\end{algorithmic}
\end{algorithm}

\section{Theoretical Analysis}
\label{sec:theory}
\subsection{Information Loss}\label{sec:infoloss}
We analyze information loss in ADKO from two sources: token compression and LM approximation. A key contribution is the formal quantification of this loss via token fidelity $\eta_k$ and a decomposition of LM error into bias and noise, providing the first rigorous link between communication compression and optimization performance. Unlike prior work on LLM compression—focused on prompt compression, weight quantization, or pretraining as lossy data encoding~\cite{nagle2024fundamental,young2025radio,conklinlearning}—ADKO studies compression within a decentralized Bayesian optimization loop. Here, tokens are generated and transmitted online, the LM acts as a reasoning module, and information loss directly impacts regret. This perspective shifts the goal from compression efficiency to principled control of performance degradation, enabling trade-offs between communication constraints and optimization accuracy.

\textbf{Token Compression: The First Loss.} Token fidelity $\eta_k$ as in Definition~\ref{def:fidelity} formalizes how much useful information survives compression from a real-valued observation into a token.
\begin{definition}[Token Fidelity]
\label{def:fidelity}
The token fidelity of token $k$ is the fraction of mutual information about
the true outcome that survives binary quantization:
$
  \eta_k
  = I\bigl(f_j(\theta_k);\; k\bigr)/H\bigl(f_j(\theta_k)\bigr)
  \;\in\; [0,1],
$
where $I(\cdot;\cdot)$ denotes mutual information and $H(\cdot)$ denotes
differential entropy under the GP posterior.
\end{definition}
Its key role is to turn communication into a quantifiable resource: high-fidelity tokens (e.g., confident outcomes far from the contextual baseline) preserve most of the decision-relevant signal, while low-fidelity tokens (near-contextual baseline, ambiguous cases) contribute little. This creates a principled notion of information efficiency—not all tokens are equally valuable, and effective collaboration depends on maintaining a high average fidelity $\bar{\eta}$. Importantly, fidelity directly controls the compression term in regret, making communication quality, not just quantity, central to convergence. This justifies mechanisms like fidelity-aware pruning as theoretically necessary, not merely heuristic. A formal result is stated in Section~\ref{fidelity_bound} to bound the token fidelity.

\textbf{LM Approximation: The Second Loss.}
LM reasoning introduces a second, orthogonal source of error: approximation of the ideal acquisition function. Decomposing this into bias (systematic misalignment with the true landscape) and noise (stochastic variability) is crucial. Bias captures persistent errors from imperfect priors or domain mismatch, while noise diminishes as more tokens accumulate, reflecting improved context. Their interaction reveals a key insight: even perfect communication (high $\eta_k$ cannot guarantee good decisions under high bias, and conversely, a well-calibrated LM cannot recover information lost to low-fidelity tokens. This separation enables targeted improvements, i.e., better prompts or fine-tuning reduce bias, while richer token memory reduces noise, making the framework both diagnosable and optimizable. We present a result for the decomposition of LM error in the following and defer proof in Section~\ref{lmerror_decomp}.
\begin{proposition}[LM Error Decomposition]
\label{prop:lmerror}
Under the assumption that the LM induces an implicit distribution
$p_{\mathrm{LM}}(y \mid \theta, \calK)$ over outcomes given context,
the LM scoring error decomposes as:
$
  R_{\mathrm{LM}}(\calK_i^t, \theta)
  = R^*(\theta) + \varepsilon_{\mathrm{bias}}(\theta) + \xi(\theta, \calK_i^t),
$
where: \emph{Systematic bias:}
    $|\varepsilon_{\mathrm{bias}}(\theta)| \leq B_R \cdot
    \TV(p_{\mathrm{LM}}(\cdot\mid\theta),\, p_{\mathrm{true}}(\cdot\mid\theta))$,
    where $B_R$ is the Lipschitz constant of $R^*$ and $\TV(\cdot,\cdot)$ is the total variation distance between the true and approximated distributions. \emph{Stochastic noise:}
    $\xi(\theta, \calK_i^t) \sim \mathcal{N}(0,\, \sigmaR^2(\calK_i^t))$ with
    $
      \sigmaR^2(\calK_i^t)
      = \frac{\sigma_0^2}{1 + \alpha_\sigma \sum_{k \in \calK_i^t} c_k \eta_k},
    $
    where $\sigma_0^2 > 0$ is the base LM variance and $\alpha_\sigma > 0$
    is an evidence-accumulation rate.
\end{proposition}

The denominator of $\sigmaR^2$ is key: it is a sum of
\emph{fidelity-weighted} advantage scores.
This means low-fidelity tokens contribute little to reducing LM noise,
directly motivating fidelity-aware pruning.



\subsection{Regret Analysis}
\label{sec:regret}
This section is dedicated to establishing theoretical analysis for ADKO, focusing primarily on the cumulative regret and how LM compression affects the regret bound. Before that, we present several auxiliary technical assumptions to characterize the analysis.

\textbf{Assumptions.}
To enable the theoretical guarantees of ADKO, the following assumptions are imposed:
\textbf{A1:} The design space $\Thetaset \subseteq \R^d$ is compact.
    $F$ is Lipschitz continuous on $\Thetaset$.
\textbf{A2:} Each $f_i$ is drawn from $\GP(0,g)$ with Mat\'ern-$5/2$ kernel.
    The Reproducing Kernel Hilbert Space (RKHS) norm satisfies $\|f_i\|_k <\infty$
    with probability $\geq 1 - \delta/(2N)$, $0<\delta<1$.
\textbf{A3:} LM error satisfies Proposition~\ref{prop:lmerror} with constants
    $B_R$, $\sigma_0$, $\alpha_\sigma$.
\textbf{A4:} The communication graph $\calG$ is connected with Fiedler value
    $\lambda_2 > 0$.
\textbf{A5:} Fidelity-aware pruning (Algorithm~\ref{alg:pruning} in Section~\ref{sec:token_pruning}) is applied at every round. Please see Section~\ref{sec:assumption_justification} for the justification of all assumptions.

With these assumptions, we are now ready to state the main result for the ADKO cumulative regret, which reveals the impact of different error sources on the algorithmic performance.
\begin{theorem}[ADKO Cumulative Regret]
\label{thm:main}
Under Assumptions \emph{(A1)--(A5)}, with probability $\geq 1 - \delta$,
the cumulative $N$-agent regret of ADKO satisfies:
\begin{equation}
\label{eq:mainbound}
  \regret_N^T
  \;\leq\;
  \underbrace{C_1 N \sqrt{T\zetaT \log(1/\delta)}}_{\text{GP term}}
  +\; \underbrace{C_2 N T B_R \TV}_{\text{LM bias}}
  +\; \underbrace{C_3 N \sqrt{T\log(1/\delta)}\;\sigma_0}_{\text{LM noise}}
  +\; \underbrace{C_4 N T (1-\etabar)}_{\text{compression}},
\end{equation}
where $C_1 = O(\beta^{1/2})$, $C_2 = O(\lambda+\gamma)$,
$C_3 = O(\lambda+\gamma)$, and
$C_4 = O\bigl((\lambda+\gamma)C_S/\lambda_2(L(\calG))\bigr)$.
\end{theorem}
The regret bound in Theorem~\ref{thm:main} additively decomposes the cumulative 
$N$-agent regret into four distinct sources, each governed by the assumptions and algorithm design. The GP term matches the standard decentralized BO rate up to the number of agents $N$ and the maximal information gain $\zetaT$. It reflects the irreducible error from GP modeling and exploration–exploitation trade-offs, and is sublinear whenever $\zetaT=o(T)$ (e.g., $\zetaT=O(\text{log}T)$for Mat\'ern kernels). The LM bias term is linear-in-$T$ term, arising from from systematic approximation errors in the LM reasoning module. To avoid linear regret, this term must vanish as $T$ grows, i.e, $\TV\to 0$, which is enforced by fidelity-aware pruning and token accumulation. LM noise term captures the stochastic component of LM approximation, with $\sigma_0$ and $\alpha_\sigma$ controlling the decay of noise as tokens accumulate (Proposition~\ref{prop:lmerror}). The $\sqrt{T}$ scaling is sublinear, so this term does not impede asymptotic convergence. Compression term is linear to quantifies the cost of token compression, where $\bar{\eta}$ is the average token fidelity. It reflects the loss of information due to quantization or pruning during communication. Sublinear regret requires $\bar{\eta}\to 1$, which Proposition~\ref{prop:pruning} guarantees under any fixed communication budget $B$ when fidelity‑aware pruning is applied. Constant $C_4$ explicitly involves $\lambda_2$, reflecting that slower mixing amplifies compression errors.
Overall, the bound shows that \textit{sublinear regret} is attainable if and only if the LM bias and compression loss are driven to zero at appropriate rates—conditions that are both necessary (by matching lower bounds in Proposition~\ref{prop:lower}) and sufficient under our modular framework, which are formally stated in Corollary~\ref{cor:sublinear} in Section~\ref{main_result}.

\section{Numerical Experiments}
\label{sec:numerical_exp}

We evaluate ADKO on two benchmarks chosen to answer two distinct questions. First, when communication is reduced by orders of magnitude, how much of the collaborative benefit does ADKO retain relative to no-communication and full-communication alternatives? Second, in realistic privacy-constrained settings where agents cannot share raw data and pursue heterogeneous local objectives, do ADKO's tokens carry the right kind of information for cross-agent transfer? We address the first with a neural architecture search (NAS) task on CIFAR-10, where weight-averaging methods are well studied and strong, and the second with a chemical reaction optimization task, where data-sharing constraints and asymmetric objectives are intrinsic to the domain. Although ADKO supports LM-guided reasoning, we use it only in the scientific discovery study: standard hyperparameter optimization tasks contain limited semantic structure for language priors to exploit, so we isolate other core contributions from ADKO in the NAS setting: decentralized token-based collaboration and information-efficient optimization.

\subsection{Case Study 1: Neural Architecture Search}\label{sec:nas}

\textbf{Experimental Setup.} We evaluate ADKO on a decentralized CIFAR-10 hyperparameter optimization task with $N=5$ agents over a $d=7$ design space combining architectural and training hyperparameters. To ensure the GP kernel correctly perceives structural distances, discrete dimensions are mapped to $K$ equally spaced grid points in $[0,1]$, so that doublings in power-of-2 parameters represent uniform coordinate steps. Following standard proxy-task evaluation protocols for resource-constrained NAS \cite{yang2023neural}, each evaluation trains a CIFAR-adapted ResNet for 30 epochs and returns test accuracy. We focus on the IID regime in the main paper, where each shard is a uniform random split across all 10 classes, and defer the non-IID scenario to
Section~\ref{sec:additional_results_nas}. 
We adopt a fully connected communication graph by default and ablate the topology choice. The baseline methods for comparison include:
Independent BO (each agent runs GP-UCB locally without communication), FedAvg BO (full parameter sharing via model-averaging Bayesian optimization), and Centralized BO (a privileged single-agent that pools raw observations across all agents). We additionally include two methods: Na\"ive Sharing (raw information shared without failure-aware filtering) and ADKO-FIFO (tokens pruned by recency rather than fidelity once memory saturates). Additional setup details appear in Section~\ref{sec:additional_nas_setup}.

\begin{figure}[htbp]
     \centering
     \begin{subfigure}[b]{0.32\textwidth}
         \centering
         \includegraphics[width=\textwidth]{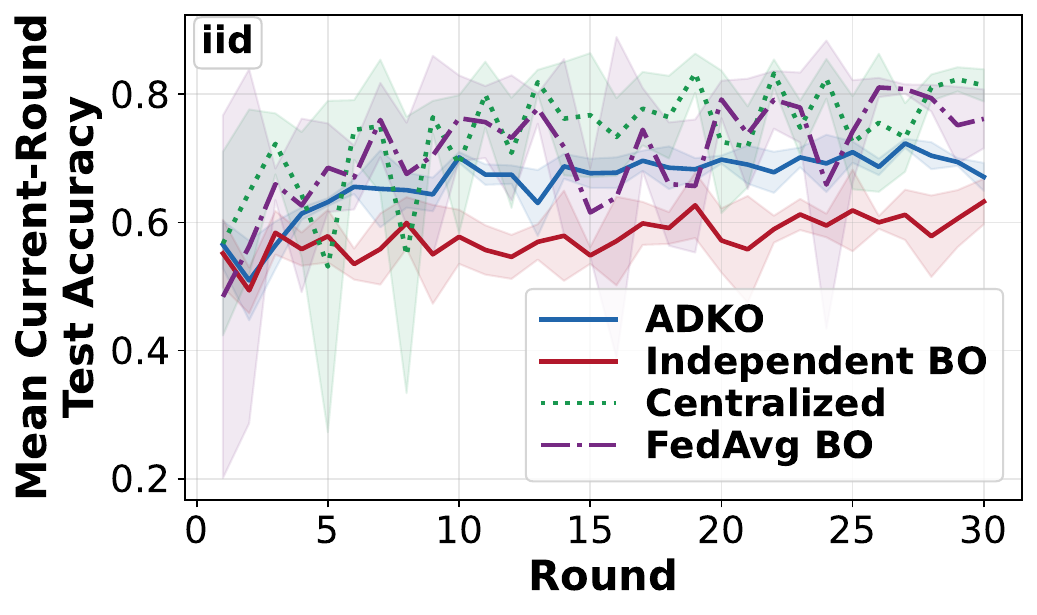}
         \caption{\textbf{Current-round mean test accuracy} comparison between ADKO and other frameworks.}
         \label{fig:nas_convergence}
     \end{subfigure}
     \hfill
     \begin{subfigure}[b]{0.32\textwidth}
         \centering
         \includegraphics[width=\textwidth]{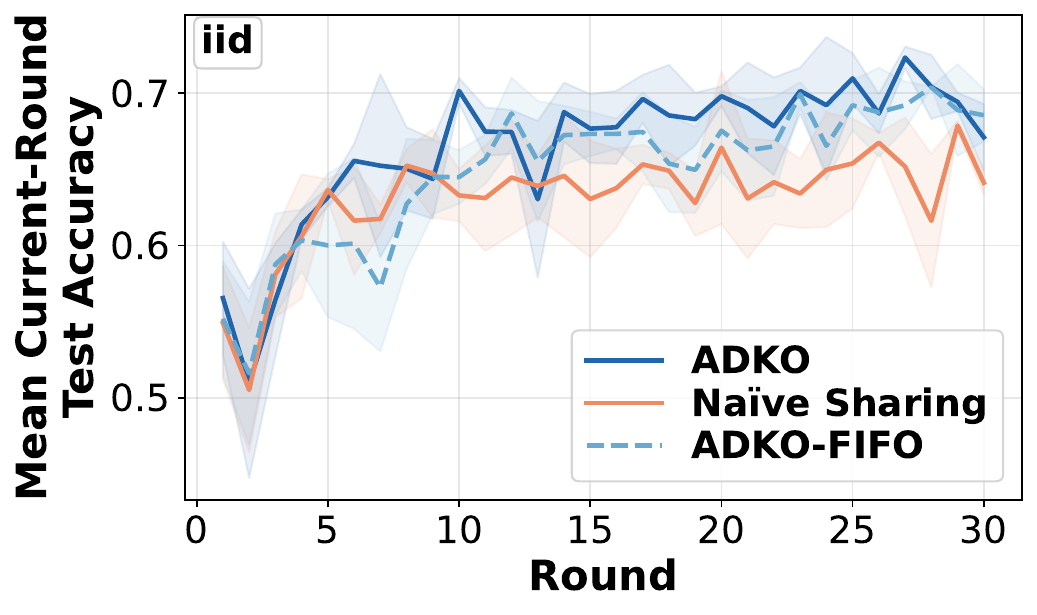}
         \caption{\textbf{Current-round mean test accuracy} comparison between ADKO and its variants.}
         \label{fig:nas_variants}
     \end{subfigure}
     \hfill
     \begin{subfigure}[b]{0.32\textwidth}
         \centering
         \includegraphics[width=\textwidth]{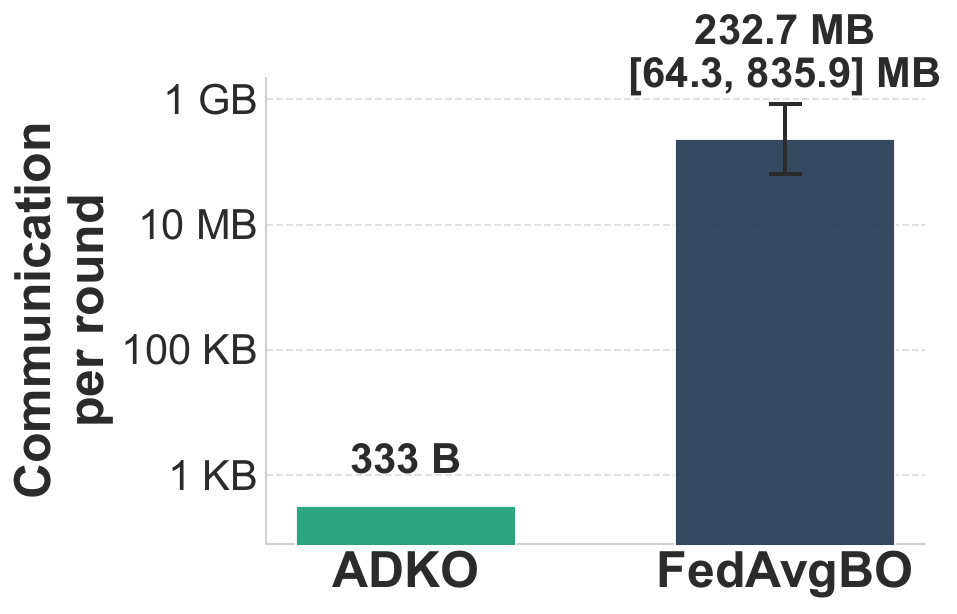}
         \caption{\textbf{Communication cost }comparison between ADKO and FedAvgBO.}
         \label{fig:nas_communication}
     \end{subfigure}
    \vspace{0.5ex}
     \caption{Results of Neural Architecture Search on CIFAR-10 with IID data.}
     \label{fig:nas}
\end{figure}

\textbf{Result Discussion.} 
Figure~\ref{fig:nas_convergence} shows mean best test accuracy over BO rounds. On CIFAR-10, ADKO consistently improves over Independent BO, indicating that compact knowledge-token exchange can provide useful cross-agent signal even without full information sharing. At the same time, ADKO remains below both the Centralized and FedAvg BO, so compact communication does not recover the full benefit of stronger coordination in this setting. FedAvg BO performs competitively with Centralized, but does so at substantially higher communication cost through full parameter sharing, as seen in Figure~\ref{fig:nas_communication}. Relative to its internal communication ablations (Figure~\ref{fig:nas_variants}), ADKO outperforms both Na\"ive Sharing and FIFO-based memory management, which supports the value of the fidelity-aware token selection. Overall, these results position ADKO as a communication-efficient alternative to heavier collaborative baselines rather than a replacement for them.

\subsection{Case Study 2: Scientific Discovery}
\label{sec:experiments}

We turn next to a setting where the assumptions ADKO targets (strict data privacy, expensive evaluations, and heterogeneous local objectives) are intrinsic to the domain rather than imposed by the experimental design. Consider four agents working on the same class of chemistry but operating under different solvent restrictions, owing to equipment, regulatory, or safety constraints. Each agent seeks the best reaction conditions \emph{within its own solvent}; there is no shared global objective. Agents cannot share raw data for IP/competitive reasons, but their structural knowledge, e.g., productive ligands, failing boron sources, robust base/coupling pairs, often transfers across solvents. 
We focus on such a non-IID scenario that is quite realistic for a real consortium of this kind. 
We defer a plausible IID scenario and additional significant results to Section~\ref{sec:additional_results_sci_dis}.

\textbf{Experimental Setup.} We use the \texttt{suzuki\_edbo} dataset from Olympus~\cite{hickman2023olympus}, which catalogs Suzuki--Miyaura cross-coupling yields across categorical reaction conditions (ligands, solvents, bases, and coupling partners). We simulate $N=4$ autonomous laboratories communicating over a fully connected graph; each agent is restricted to a different solvent and aims to maximize yield within its own solvent slice. 
Additional details on baselines, hyperparameters, and LLM integration appear in Section~\ref{sec:additional_sci_dis_setup}.
We evaluate ADKO in two configurations that test whether its reasoning score, $\mu + \beta\sigma + \lambda G - \gamma\Lambda$, is sufficient on its own. The score is privacy-aware but chemistry-blind: it can rank candidates from compact distributional statistics, but it cannot represent that an iodide/BPin/alkoxide motif productive in DMF should plausibly transfer to MeCN. \emph{ADKO} runs the reasoning over the entire $\sim$3{,}696-condition grid each round, isolating the performance of the score itself.
\emph{ADKO‑LLM} uses just 10 LLM‑proposed candidates per round ($<$1\% of the space), leveraging chemical intuition that a working solvent condition transfers across solvents.

\textbf{Baselines.} We compare against Centralized BO (a single global GP trained on pooled raw observations, with candidate assignments restricted to each agent's solvent slice; the privacy-violating ceiling), Independent BO (each agent runs GP-UCB on its private history with no inter-agent communication), Federated Thompson Sampling (FTS, parameter-sharing via shared random Fourier feature weights, adapted to the discrete Suzuki space), and PoE GP-UCB (model-sharing via aggregation of dense predictive posteriors across the candidate grid). All baselines evaluate the full candidate space each round; only ADKO-LLM operates under the 10-candidate restriction.
\begin{figure}[htbp]
    \centering
    
    \begin{minipage}[t]{0.55\linewidth}
        \centering
        \includegraphics[width=\linewidth]{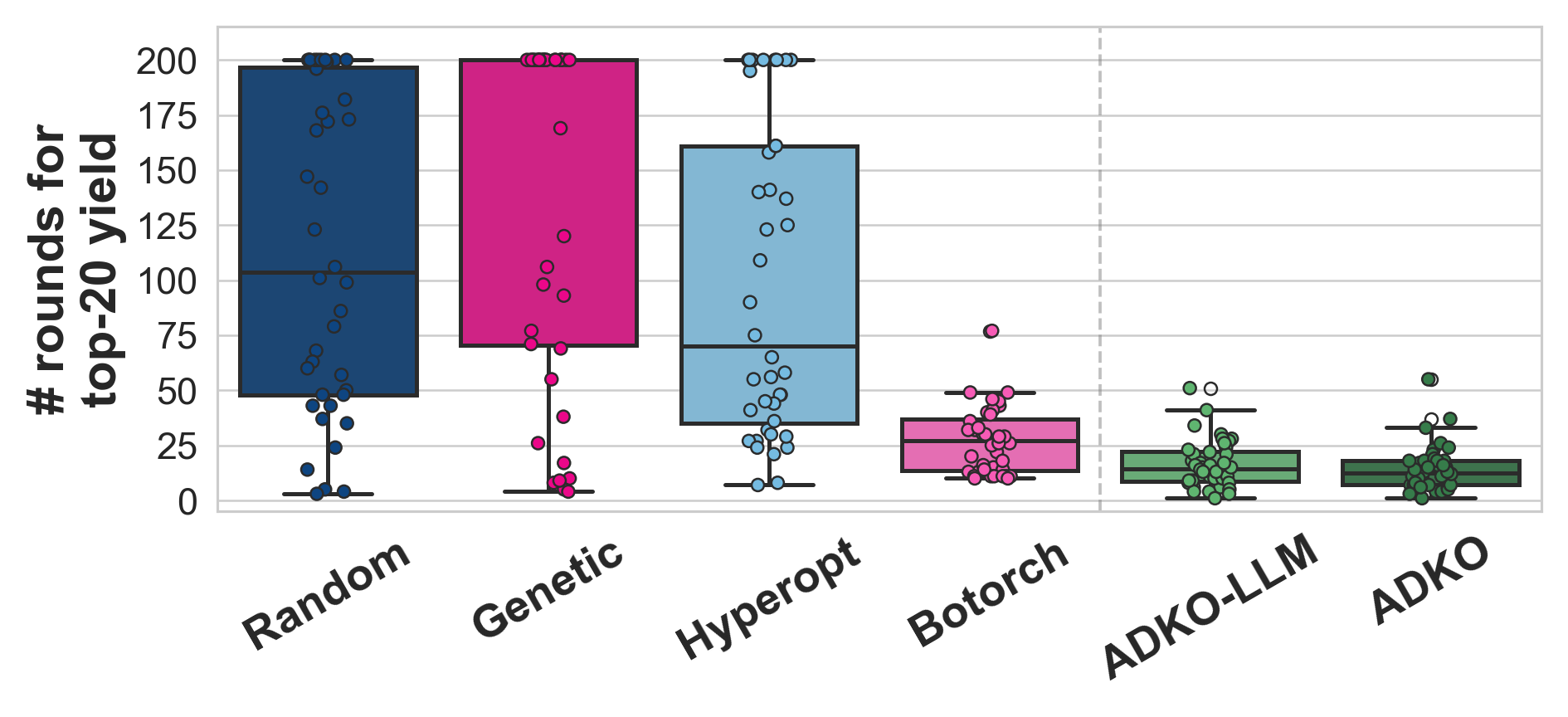}
        
        \vspace{0.5ex}
        {\small (a) Comparison to Olympus baselines}
        \label{fig:olympus_suzuki}
    \end{minipage}
    \hfill
    \begin{minipage}[t]{0.42\linewidth}
        \centering
        \includegraphics[width=\linewidth]{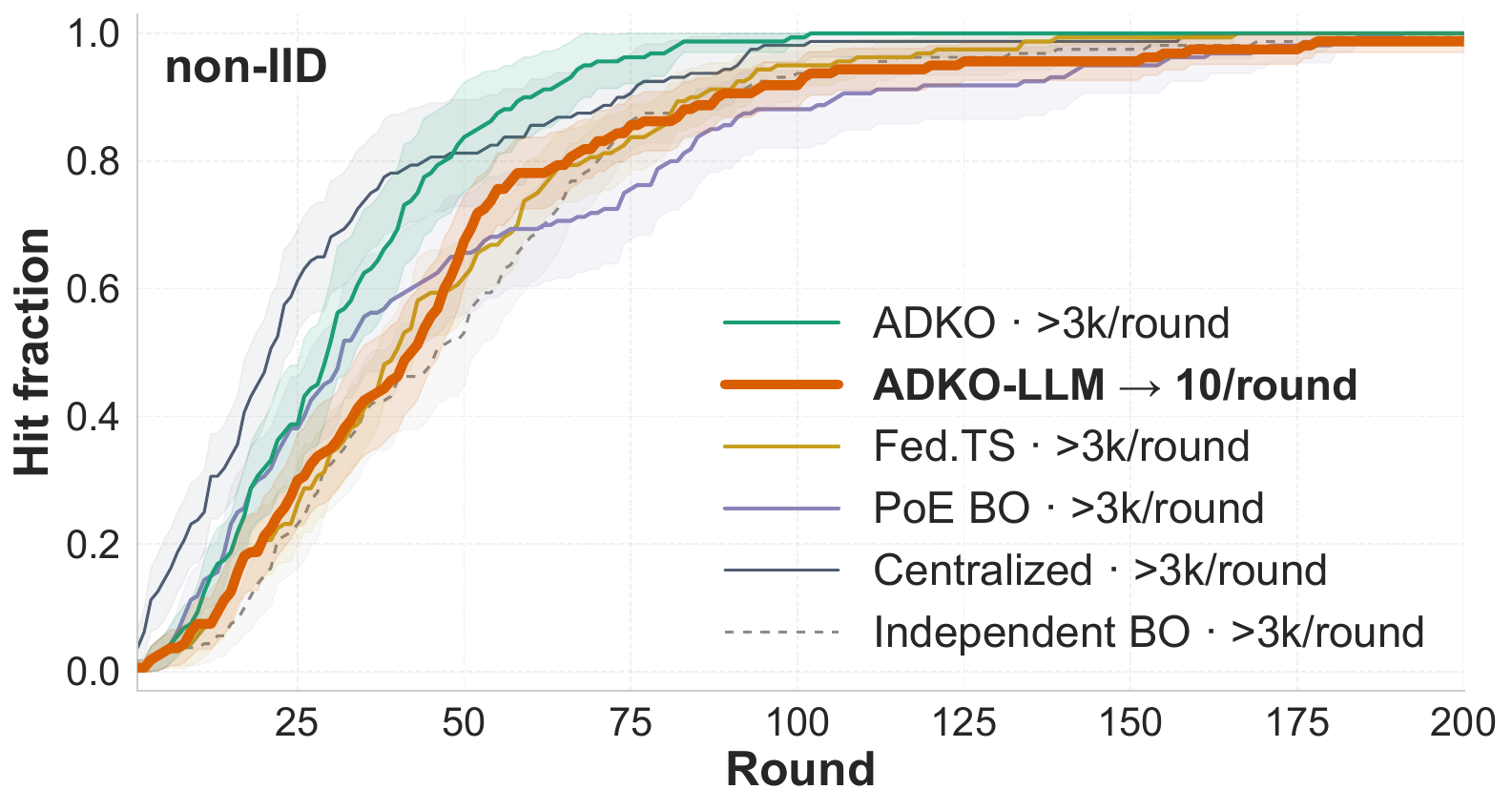}
        \vspace{0.5ex}
        {\small (b) Hit fraction under non-IID solvent restrictions}
        \label{fig:sci_dis_main}
    \end{minipage}

    \caption{
    Suzuki benchmark results. 
    (a) ADKO compared to Olympus baselines~\cite{hickman2023olympus}. 
    (b) Hit fraction over evaluations under non-IID solvent restrictions. 
    All methods evaluate the full $\sim$3{,}696$-$candidate space per round except ADKO-LLM, which evaluates 10 LLM-proposed candidates per round (under 1\% of the search space).
    }
    \label{fig:suzuki_convergence}
\end{figure}

\textbf{Metric.} We report \emph{hit fraction}: the fraction of agents, averaged across random seeds, that have identified one of their top-3 conditions within their own solvent by a given evaluation budget. Success in this setting is defined locally---an agent restricted to MeOH cannot benefit from a discovery made only in DMF---so aggregate metrics like average yield or time-to-best-yield would let agents operating in easier solvents dominate, obscuring whether the harder-constrained agents are succeeding at all. Hit fraction treats every local objective equitably. 

\textbf{Main Results.} Figure~\ref{fig:suzuki_convergence}(b) shows that, ADKO improves over every distributed baseline and trails only the Centralized baseline, supporting the claim that compact tokens carry decision-relevant information about chemical motifs and failure patterns even when local objectives are misaligned across labs. ADKO-LLM also matches or exceeds the best distributed baselines.
Notably, ADKO-LLM keeps pace with methods that scan the full grid because the LLM filters out points the score would have ranked highly but that fail for reasons outside the score's representational scope, while admitting cross-solvent transfers the score cannot propose.
A DMF‑restricted agent finds an iodide/BPin/LiO
LiO\textit{i}-Pr/phosphine motif (96.28\%). A MeCN agent transfers it successfully (99.15\% by round 5), and a MeOH agent refines it to 100\% by round 10, all without raw data, using only compact tokens. In IID settings, this prior becomes counterproductive: the reasoning score is already well calibrated, and gating it through 10 LLM picks substitutes a coarse semantic prior for a reliable quantitative one (See Section~\ref{sec:additional_results_sci_dis}).
Figure~\ref{fig:suzuki_convergence}(a) reports evaluations to a top-20 yield against the standard Olympus benchmark suite (Random, Genetic, Hyperopt, Botorch). The Olympus baselines are single-agent optimizers, and Botorch corresponds to a single Independent BO agent in our setup. 
ADKO with or without LLM
reaches a top-20 yield in substantially fewer evaluations than any single-agent method, validating the effectiveness and efficiency of our proposed methods.

\begin{wrapfigure}{r}{0.55\linewidth}
    \centering
    \includegraphics[width=1\linewidth]{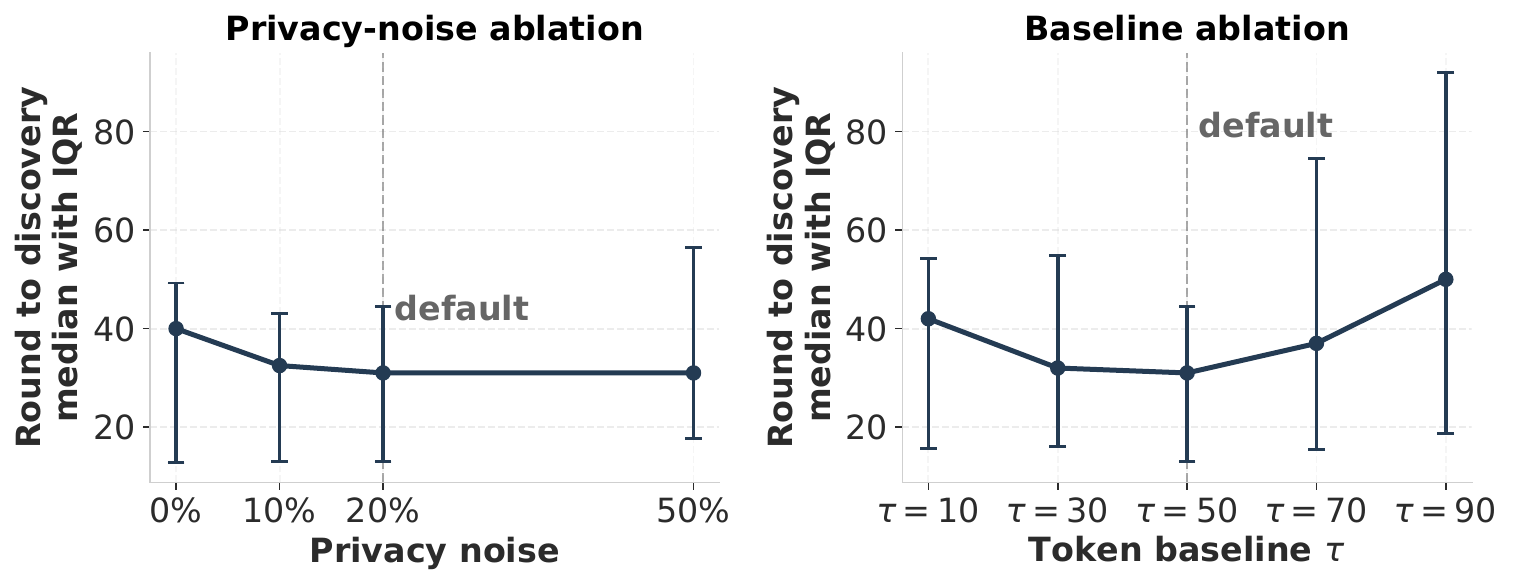}
    \caption{Noise and contextual-baseline sweep on the \texttt{suzuki\_edbo} task. Lower is better.}
    \label{fig:sci_dis_ablations}
    \vspace{-7pt}
\end{wrapfigure}

\textbf{Ablations.} The ablations in Figure~\ref{fig:sci_dis_ablations} suggest that 
ADKO maintains a decent privacy-utility tradeoff as 
adding moderate privacy noise to the 5-dimensional token embedding does not hurt performance. 
Intermediate noise levels outperform the noiseless setting, consistent with slight corruption acting as a decentralized exploration regularizer that reduces premature lock-in to locally attractive regions. The ablation on the contextual baseline $\tau$ justifies our empirical design in this work, showing $\tau=50$ leads to the rapid discovery of top-3 yield since it retains enough \texttt{SUCCESS} structure for fine-grained optimization.


\textbf{Limitations.}
First, for highly novel discovery tasks, LLM priors may be systematically wrong, in which case ADKO-LLM can underperform Independent BO.
Second, the GP has $O(n^3)$ training complexity per agent. Sparse GP approximations~\cite{pmlr-v5-titsias09a} could address this at the cost of weakening (A2).
Third, our privacy guarantee (Constraint~\ref{con:privacy}) is informal. A formal differentially private guarantee would require adding calibrated noise to token encoding, reducing $\eta_k$ and worsening the regret bound; a precise privacy--regret tradeoff is an open problem.
Also, ADKO assumes honest token encoding. Adversarial agents injecting misleading tokens could corrupt $G_i$ and $\Lambda_i$; robust extensions are needed before deployment in untrusted settings.

\section{Conclusion and Broader Impact}
\label{sec:conclusion}
We present ADKO, a unified framework for decentralized knowledge optimization that integrates private GP surrogates, knowledge-token communication, and optional LM reasoning. Our key theoretical result is a four-term regret decomposition separating GP error, LM bias, LM noise, and token compression, each tied to a clear design lever. We also introduce fidelity-aware pruning, a principled mechanism for maintaining high-quality token memory under constraints. Experiments on neural architecture search and scientific discovery validate both the theory and practical gains. More broadly, ADKO provides a blueprint for combining probabilistic modeling, foundation models, and distributed coordination into a scalable, privacy-preserving system. It applies to domains such as autonomous labs, distributed hyperparameter tuning, and multi-institution optimization under data constraints, with a built-in fairness objective across agents. 

\clearpage
\bibliographystyle{unsrt}  
\bibliography{references}  

@inproceedings{movckus1974bayesian,
  title={On Bayesian methods for seeking the extremum},
  author={Mo{\v{c}}kus, Jonas},
  booktitle={IFIP Technical Conference on Optimization Techniques},
  pages={400--404},
  year={1974},
  organization={Springer}
}

@article{jones1998efficient,
  title={Efficient global optimization of expensive black-box functions},
  author={Jones, Donald R and Schonlau, Matthias and Welch, William J},
  journal={Journal of Global optimization},
  volume={13},
  number={4},
  pages={455--492},
  year={1998},
  publisher={Springer}
}

@article{srinivas2009gaussian,
  title={Gaussian process optimization in the bandit setting: No regret and experimental design},
  author={Srinivas, Niranjan and Krause, Andreas and Kakade, Sham M and Seeger, Matthias},
  journal={arXiv preprint arXiv:0912.3995},
  year={2009}
}

@inproceedings{gonzalez2016batch,
  title={Batch Bayesian optimization via local penalization},
  author={Gonz{\'a}lez, Javier and Dai, Zhenwen and Hennig, Philipp and Lawrence, Neil},
  booktitle={Artificial intelligence and statistics},
  pages={648--657},
  year={2016},
  organization={PMLR}
}

@inproceedings{kusner2015differentially,
  title={Differentially private Bayesian optimization},
  author={Kusner, Matt and Gardner, Jacob and Garnett, Roman and Weinberger, Kilian},
  booktitle={International conference on machine learning},
  pages={918--927},
  year={2015},
  organization={PMLR}
}

@inproceedings{mcmahan2017communication,
  title={Communication-efficient learning of deep networks from decentralized data},
  author={McMahan, Brendan and Moore, Eider and Ramage, Daniel and Hampson, Seth and y Arcas, Blaise Aguera},
  booktitle={Artificial intelligence and statistics},
  pages={1273--1282},
  year={2017},
  organization={Pmlr}
}

@inproceedings{yang2023large,
  title={Large language models as optimizers},
  author={Yang, Chengrun and Wang, Xuezhi and Lu, Yifeng and Liu, Hanxiao and Le, Quoc V and Zhou, Denny and Chen, Xinyun},
  booktitle={The Twelfth International Conference on Learning Representations},
  year={2023}
}

@article{brochu2010tutorial,
  title={A tutorial on Bayesian optimization of expensive cost functions, with application to active user modeling and hierarchical reinforcement learning},
  author={Brochu, Eric and Cora, Vlad M and De Freitas, Nando},
  journal={arXiv preprint arXiv:1012.2599},
  year={2010}
}

@article{alistarh2017qsgd,
  title={QSGD: Communication-efficient SGD via gradient quantization and encoding},
  author={Alistarh, Dan and Grubic, Demjan and Li, Jerry and Tomioka, Ryota and Vojnovic, Milan},
  journal={Advances in neural information processing systems},
  volume={30},
  year={2017}
}

@article{auer2002using,
  title={Using confidence bounds for exploitation-exploration trade-offs},
  author={Auer, Peter},
  journal={Journal of machine learning research},
  volume={3},
  number={Nov},
  pages={397--422},
  year={2002}
}

@article{azimi2012hybrid,
  title={Hybrid batch Bayesian optimization},
  author={Azimi, Javad and Jalali, Ali and Fern, Xiaoli},
  journal={arXiv preprint arXiv:1202.5597},
  year={2012}
}

@article{boiko2023emergent,
  title={Emergent autonomous scientific research capabilities of large language models},
  author={Boiko, Daniil A and MacKnight, Robert and Gomes, Gabe},
  journal={arXiv preprint arXiv:2304.05332},
  year={2023}
}

@article{boyd2006randomized,
  title={Randomized gossip algorithms},
  author={Boyd, Stephen and Ghosh, Arpita and Prabhakar, Balaji and Shah, Devavrat},
  journal={IEEE transactions on information theory},
  volume={52},
  number={6},
  pages={2508--2530},
  year={2006},
  publisher={IEEE}
}

@article{bran2023chemcrow,
  title={Chemcrow: Augmenting large-language models with chemistry tools},
  author={Bran, Andres M and Cox, Sam and Schilter, Oliver and Baldassari, Carlo and White, Andrew D and Schwaller, Philippe},
  journal={arXiv preprint arXiv:2304.05376},
  year={2023}
}

@article{chen2023evoprompting,
  title={Evoprompting: Language models for code-level neural architecture search},
  author={Chen, Angelica and Dohan, David and So, David},
  journal={Advances in neural information processing systems},
  volume={36},
  pages={7787--7817},
  year={2023}
}

@article{desautels2014parallelizing,
  title={Parallelizing exploration-exploitation tradeoffs in gaussian process bandit optimization},
  author={Desautels, Thomas and Krause, Andreas and Burdick, Joel W},
  journal={The Journal of Machine Learning Research},
  volume={15},
  number={1},
  pages={3873--3923},
  year={2014},
  publisher={JMLR. org}
}

@article{eriksson2019scalable,
  title={Scalable global optimization via local Bayesian optimization},
  author={Eriksson, David and Pearce, Michael and Gardner, Jacob and Turner, Ryan D and Poloczek, Matthias},
  journal={Advances in neural information processing systems},
  volume={32},
  year={2019}
}

@incollection{ginsbourger2010kriging,
  title={Kriging is well-suited to parallelize optimization},
  author={Ginsbourger, David and Le Riche, Rodolphe and Carraro, Laurent},
  booktitle={Computational intelligence in expensive optimization problems},
  pages={131--162},
  year={2010},
  publisher={Springer}
}

@article{goetz2019active,
  title={Active federated learning},
  author={Goetz, Jack and Malik, Kshitiz and Bui, Duc and Moon, Seungwhan and Liu, Honglei and Kumar, Anuj},
  journal={arXiv preprint arXiv:1909.12641},
  year={2019}
}

@article{liu2024large,
  title={Large language models to enhance bayesian optimization},
  author={Liu, Tennison and Astorga, Nicol{\'a}s and Seedat, Nabeel and van der Schaar, Mihaela},
  journal={arXiv preprint arXiv:2402.03921},
  year={2024}
}

@inproceedings{kandasamy2017multi,
  title={Multi-fidelity bayesian optimisation with continuous approximations},
  author={Kandasamy, Kirthevasan and Dasarathy, Gautam and Schneider, Jeff and P{\'o}czos, Barnab{\'a}s},
  booktitle={International conference on machine learning},
  pages={1799--1808},
  year={2017},
  organization={PMLR}
}

@inproceedings{karimireddy2020scaffold,
  title={Scaffold: Stochastic controlled averaging for federated learning},
  author={Karimireddy, Sai Praneeth and Kale, Satyen and Mohri, Mehryar and Reddi, Sashank and Stich, Sebastian and Suresh, Ananda Theertha},
  booktitle={International conference on machine learning},
  pages={5132--5143},
  year={2020},
  organization={PMLR}
}

@article{li2020federated,
  title={Federated learning: Challenges, methods, and future directions},
  author={Li, Tian and Sahu, Anit Kumar and Talwalkar, Ameet and Smith, Virginia},
  journal={IEEE signal processing magazine},
  volume={37},
  number={3},
  pages={50--60},
  year={2020},
  publisher={IEEE}
}

@article{lowe2017multi,
  title={Multi-agent actor-critic for mixed cooperative-competitive environments},
  author={Lowe, Ryan and Wu, Yi I and Tamar, Aviv and Harb, Jean and Pieter Abbeel, OpenAI and Mordatch, Igor},
  journal={Advances in neural information processing systems},
  volume={30},
  year={2017}
}

@article{jiang2017collaborative,
  title={Collaborative deep learning in fixed topology networks},
  author={Jiang, Zhanhong and Balu, Aditya and Hegde, Chinmay and Sarkar, Soumik},
  journal={Advances in Neural Information Processing Systems},
  volume={30},
  year={2017}
}

@article{muller2021transformers,
  title={Transformers can do bayesian inference},
  author={M{\"u}ller, Samuel and Hollmann, Noah and Arango, Sebastian Pineda and Grabocka, Josif and Hutter, Frank},
  journal={arXiv preprint arXiv:2112.10510},
  year={2021}
}

@article{nedic2009distributed,
  title={Distributed subgradient methods for multi-agent optimization},
  author={Nedic, Angelia and Ozdaglar, Asuman},
  journal={IEEE Transactions on automatic control},
  volume={54},
  number={1},
  pages={48--61},
  year={2009},
  publisher={IEEE}
}

@article{romera2024mathematical,
  title={Mathematical discoveries from program search with large language models},
  author={Romera-Paredes, Bernardino and Barekatain, Mohammadamin and Novikov, Alexander and Balog, Matej and Kumar, M Pawan and Dupont, Emilien and Ruiz, Francisco JR and Ellenberg, Jordan S and Wang, Pengming and Fawzi, Omar and others},
  journal={Nature},
  volume={625},
  number={7995},
  pages={468--475},
  year={2024},
  publisher={Nature Publishing Group UK London}
}

@inproceedings{sim2021collaborative,
  title={Collaborative Bayesian optimization with fair regret},
  author={Sim, Rachael Hwee Ling and Zhang, Yehong and Low, Bryan Kian Hsiang and Jaillet, Patrick},
  booktitle={International Conference on Machine Learning},
  pages={9691--9701},
  year={2021},
  organization={PMLR}
}

@inproceedings{NIPS2012_05311655,
 author = {Snoek, Jasper and Larochelle, Hugo and Adams, Ryan P},
 booktitle = {Advances in Neural Information Processing Systems},
 editor = {F. Pereira and C.J. Burges and L. Bottou and K.Q. Weinberger},
 pages = {},
 publisher = {Curran Associates, Inc.},
 title = {Practical Bayesian Optimization of Machine Learning Algorithms},
 url = {https://proceedings.neurips.cc/paper_files/paper/2012/file/05311655a15b75fab86956663e1819cd-Paper.pdf},
 volume = {25},
 year = {2012}
}

@InProceedings{pmlr-v5-titsias09a,
  title = 	 {Variational Learning of Inducing Variables in Sparse Gaussian Processes},
  author = 	 {Titsias, Michalis},
  booktitle = 	 {Proceedings of the Twelfth International Conference on Artificial Intelligence and Statistics},
  pages = 	 {567--574},
  year = 	 {2009},
  editor = 	 {van Dyk, David and Welling, Max},
  volume = 	 {5},
  series = 	 {Proceedings of Machine Learning Research},
  address = 	 {Hilton Clearwater Beach Resort, Clearwater Beach, Florida USA},
  month = 	 {16--18 Apr},
  publisher =    {PMLR},
  url = 	 {https://proceedings.mlr.press/v5/titsias09a.html},
}

@InProceedings{pmlr-v130-vakili21a,
  title = 	 { On Information Gain and Regret Bounds in Gaussian Process Bandits },
  author =       {Vakili, Sattar and Khezeli, Kia and Picheny, Victor},
  booktitle = 	 {Proceedings of The 24th International Conference on Artificial Intelligence and Statistics},
  pages = 	 {82--90},
  year = 	 {2021},
  editor = 	 {Banerjee, Arindam and Fukumizu, Kenji},
  volume = 	 {130},
  series = 	 {Proceedings of Machine Learning Research},
  month = 	 {13--15 Apr},
  publisher =    {PMLR},
  url = 	 {https://proceedings.mlr.press/v130/vakili21a.html},
}

@inproceedings{NEURIPS2022_9c1535a0,
 author = {Yu, Chao and Velu, Akash and Vinitsky, Eugene and Gao, Jiaxuan and Wang, Yu and Bayen, Alexandre and WU, YI},
 booktitle = {Advances in Neural Information Processing Systems},
 editor = {S. Koyejo and S. Mohamed and A. Agarwal and D. Belgrave and K. Cho and A. Oh},
 pages = {24611--24624},
 publisher = {Curran Associates, Inc.},
 title = {The Surprising Effectiveness of PPO in Cooperative Multi-Agent Games},
 url = {https://proceedings.neurips.cc/paper_files/paper/2022/file/9c1535a02f0ce079433344e14d910597-Paper-Datasets_and_Benchmarks.pdf},
 volume = {35},
 year = {2022}
}

@article{wang2016bayesian,
author = {Wang, Ziyu and Hutter, Frank and Zoghi, Masrour and Matheson, David and De Freitas, Nando},
title = {Bayesian optimization in a billion dimensions via random embeddings},
year = {2016},
issue_date = {January 2016},
publisher = {AI Access Foundation},
address = {El Segundo, CA, USA},
volume = {55},
number = {1},
issn = {1076-9757},
journal = {J. Artif. Int. Res.},
month = jan,
pages = {361–387},
numpages = {27}
}

@ARTICLE{10542323,
  author={Yuan, Liangqi and Wang, Ziran and Sun, Lichao and Yu, Philip S. and Brinton, Christopher G.},
  journal={IEEE Internet of Things Journal}, 
  title={Decentralized Federated Learning: A Survey and Perspective}, 
  year={2024},
  volume={11},
  number={21},
  pages={34617-34638},
  doi={10.1109/JIOT.2024.3407584}}

@ARTICLE{10251949,
  author={Martínez Beltrán, Enrique Tomás and Pérez, Mario Quiles and Sánchez, Pedro Miguel Sánchez and Bernal, Sergio López and Bovet, Gérôme and Pérez, Manuel Gil and Pérez, Gregorio Martínez and Celdrán, Alberto Huertas},
  journal={IEEE Communications Surveys \& Tutorials}, 
  title={Decentralized Federated Learning: Fundamentals, State of the Art, Frameworks, Trends, and Challenges}, 
  year={2023},
  volume={25},
  number={4},
  pages={2983-3013},
  doi={10.1109/COMST.2023.3315746}}

@ARTICLE{10197242,
  author={Huang, Honglan and Shi, Wei and Feng, Yanghe and Niu, Chaoyue and Cheng, Guangquan and Huang, Jincai and Liu, Zhong},
  journal={IEEE Transactions on Neural Networks and Learning Systems}, 
  title={Active Client Selection for Clustered Federated Learning}, 
  year={2024},
  volume={35},
  number={11},
  pages={16424-16438},
  doi={10.1109/TNNLS.2023.3294295}}

@ARTICLE{10833800,
  author={Taghavifar, Hamid and Hu, Chuan and Wei, Chongfeng and Mohammadzadeh, Ardashir and Zhang, Chunwei},
  journal={IEEE Transactions on Automation Science and Engineering}, 
  title={Behaviorally-Aware Multi-Agent RL With Dynamic Optimization for Autonomous Driving}, 
  year={2025},
  volume={22},
  number={},
  pages={10672-10683},
  doi={10.1109/TASE.2025.3527327}}

@article{BRAHMACHARY2025129272,
title = {Large language model-based evolutionary optimizer: Reasoning with elitism},
journal = {Neurocomputing},
volume = {622},
pages = {129272},
year = {2025},
issn = {0925-2312},
doi = {https://doi.org/10.1016/j.neucom.2024.129272},
url = {https://www.sciencedirect.com/science/article/pii/S0925231224020435},
author = {Shuvayan Brahmachary and Subodh M. Joshi and Aniruddha Panda and Kaushik Koneripalli and Arun Kumar Sagotra and Harshil Patel and Ankush Sharma and Ameya D. Jagtap and Kaushic Kalyanaraman},
}

@misc{ren2026scientificintelligencesurveyllmbased,
      title={Towards Scientific Intelligence: A Survey of LLM-based Scientific Agents}, 
      author={Shuo Ren and Can Xie and Pu Jian and Zhenjiang Ren and Chunlin Leng and Jiajun Zhang},
      year={2026},
      eprint={2503.24047},
      archivePrefix={arXiv},
      primaryClass={cs.AI},
      url={https://arxiv.org/abs/2503.24047}, 
}

@article{Li_Li_Wu_Liao_HAO_Shao_Xu_2026, title={AgentSwift: Efficient LLM Agent Design via Value-Guided Hierarchical Search}, volume={40}, url={https://ojs.aaai.org/index.php/AAAI/article/view/40453}, DOI={10.1609/aaai.v40i38.40453}, number={38}, journal={Proceedings of the AAAI Conference on Artificial Intelligence}, author={Li, Yu and Li, Lehui and Wu, Zhihao and Liao, Qingmin and HAO, Jianye and Shao, Kun and Xu, Fengli}, year={2026}, month={Mar.}, pages={31843-31851} }

@misc{hu2026oscagentacceleratingdiscoveryorganic,
      title={OSCAgent: Accelerating the Discovery of Organic Solar Cells with LLM Agents}, 
      author={Zhaolin Hu and Zhiliang Wu and Hehe Fan and Yi Yang},
      year={2026},
      eprint={2602.04510},
      archivePrefix={arXiv},
      primaryClass={cs.CE},
      url={https://arxiv.org/abs/2602.04510}, 
}

@inproceedings{conklinlearning,
  title={Learning is Forgetting; LLM Training As Lossy Compression},
  author={Conklin, Henry and Hosking, Tom and Yi-Chern, Tan and Cohen, Jonathan D and Leslie, Sarah-Jane and Griffiths, Thomas L and Bartolo, Max and Goldfarb-Tarrant, Seraphina},
  booktitle={The Fourteenth International Conference on Learning Representations},
  year={2026},
}

@article{nagle2024fundamental,
  title={Fundamental limits of prompt compression: A rate-distortion framework for black-box language models},
  author={Nagle, Alliot and Girish, Adway and Bondaschi, Marco and Gastpar, Michael and Makkuva, Ashok Vardhan and Kim, Hyeji},
  journal={Advances in Neural Information Processing Systems},
  volume={37},
  pages={94934--94970},
  year={2024}
}

@article{young2025radio,
  title={Radio: rate-distortion optimization for large language model compression},
  author={Young, Sean I},
  journal={arXiv preprint arXiv:2505.03031},
  year={2025}
}

@article{ahmed2020combining,
  title={Combining Bayesian optimization and Lipschitz optimization},
  author={Ahmed, Mohamed Osama and Vaswani, Sharan and Schmidt, Mark},
  journal={Machine Learning},
  volume={109},
  number={1},
  pages={79--102},
  year={2020},
  publisher={Springer}
}

@inproceedings{jin2025gridmind,
  title={GridMind: LLMs-powered agents for power system analysis and operations},
  author={Jin, Hongwei and Kim, Kibaek and Kwon, Jonghwan},
  booktitle={Proceedings of the SC'25 Workshops of the International Conference for High Performance Computing, Networking, Storage and Analysis},
  pages={560--568},
  year={2025}
}

@article{swanson2025virtual,
  title={The Virtual Lab of AI agents designs new SARS-CoV-2 nanobodies},
  author={Swanson, Kyle and Wu, Wesley and Bulaong, Nash L and Pak, John E and Zou, James},
  journal={Nature},
  volume={646},
  number={8085},
  pages={716--723},
  year={2025},
  publisher={Nature Publishing Group UK London}
}

@inproceedings{dow2011prototyping,
  title={Prototyping dynamics: sharing multiple designs improves exploration, group rapport, and results},
  author={Dow, Steven and Fortuna, Julie and Schwartz, Dan and Altringer, Beth and Schwartz, Daniel and Klemmer, Scott},
  booktitle={Proceedings of the SIGCHI conference on human factors in computing systems},
  pages={2807--2816},
  year={2011}
}

@inproceedings{solanas2004coordinated,
  title={Coordinated multi-robot exploration through unsupervised clustering of unknown space},
  author={Solanas, Agusti and Garcia, Miguel Angel},
  booktitle={2004 IEEE/RSJ International Conference on Intelligent Robots and Systems (IROS)(IEEE Cat. No. 04CH37566)},
  volume={1},
  pages={717--721},
  year={2004},
  organization={IEEE}
}

@inproceedings{hamiti2024data,
  title={A Data Space infrastructure supporting the integration of clinical data nodes and cancer registries to improve personalized medicine},
  author={Hamiti, Florim and Breidenbach, Martin and Heiba, Naguib and Guluzade, Aynur and Mohamad, Yehya and Velasco, Carlos A},
  booktitle={Proceedings of the 11th International Conference on Software Development and Technologies for Enhancing Accessibility and Fighting Info-exclusion},
  pages={252--259},
  year={2024}
}

@inproceedings{sparks2015automating,
  title={Automating model search for large scale machine learning},
  author={Sparks, Evan R and Talwalkar, Ameet and Haas, Daniel and Franklin, Michael J and Jordan, Michael I and Kraska, Tim},
  booktitle={Proceedings of the Sixth ACM Symposium on Cloud Computing},
  pages={368--380},
  year={2015}
}

@article{yue2025collaborative,
  title={Collaborative and distributed bayesian optimization via consensus},
  author={Yue, Xubo and Liu, Yang and Berahas, Albert S and Johnson, Blake N and Al Kontar, Raed},
  journal={IEEE Transactions on Automation Science and Engineering},
  volume={22},
  pages={11343--11355},
  year={2025},
  publisher={IEEE}
}

@article{yang2023neural,
  title={Neural architecture search for resource constrained hardware devices: A survey},
  author={Yang, Yongjia and Zhan, Jinyu and Jiang, Wei and Jiang, Yucheng and Yu, Antai},
  journal={IET Cyber-Physical Systems: Theory \& Applications},
  volume={8},
  number={3},
  pages={149--159},
  year={2023},
  publisher={Wiley Online Library}
}

@inproceedings{khodak2020weight,
  title={Weight sharing for hyperparameter optimization in federated learning},
  author={Khodak, Mikhail and Li, Tian and Li, Liam and Balcan, M and Smith, Virginia and Talwalkar, Ameet},
  booktitle={Int. Workshop on Federated Learning for User Privacy and Data Confidentiality in Conjunction with ICML},
  volume={2020},
  year={2020}
}

@inproceedings{dai2020federated,
 author = {Dai, Zhongxiang and Low, Bryan Kian Hsiang and Jaillet, Patrick},
 booktitle = {Advances in Neural Information Processing Systems},
 editor = {H. Larochelle and M. Ranzato and R. Hadsell and M.F. Balcan and H. Lin},
 pages = {9687--9699},
 publisher = {Curran Associates, Inc.},
 title = {Federated Bayesian Optimization via Thompson Sampling},
 url = {https://proceedings.neurips.cc/paper_files/paper/2020/file/6dfe08eda761bd321f8a9b239f6f4ec3-Paper.pdf},
 volume = {33},
 year = {2020}
}

@ARTICLE{hinton2002training,
  author={Hinton, Geoffrey E.},
  journal={Neural Computation}, 
  title={Training Products of Experts by Minimizing Contrastive Divergence}, 
  year={2002},
  volume={14},
  number={8},
  pages={1771-1800},
  doi={10.1162/089976602760128018},
  ISSN={0899-7667},
  month={Aug},}

@misc{cao2014generalized,
Author = {Yanshuai Cao and David J. Fleet},
Title = {Generalized Product of Experts for Automatic and Principled Fusion of Gaussian Process Predictions},
Year = {2014},
Eprint = {arXiv:1410.7827},
}

@article{
hickman2023olympus,
author = {Riley Hickman  and Priyansh Parakh  and Austin Cheng  and Qianxiang Ai  and Joshua Schrier  and Matteo Aldeghi  and Alán Aspuru-Guzik },
title = {Olympus, enhanced: benchmarking mixed-parameter and multi-objective optimization in chemistry and materials science},
journal = {ChemRxiv},
volume = {2023},
number = {0518},
pages = {},
year = {2023},
doi = {10.26434/chemrxiv-2023-74w8d},
URL = {https://chemrxiv.org/doi/abs/10.26434/chemrxiv-2023-74w8d},
eprint = {https://chemrxiv.org/doi/pdf/10.26434/chemrxiv-2023-74w8d},
}

@article{balandat2020botorch,
  title={BoTorch: A framework for efficient Monte-Carlo Bayesian optimization},
  author={Balandat, Maximilian and Karrer, Brian and Jiang, Daniel and Daulton, Samuel and Letham, Ben and Wilson, Andrew G and Bakshy, Eytan},
  journal={Advances in neural information processing systems},
  volume={33},
  pages={21524--21538},
  year={2020}
}

@inproceedings{kennedy1995particle,
  title={Particle swarm optimization},
  author={Kennedy, James and Eberhart, Russell},
  booktitle={Proceedings of ICNN'95-international conference on neural networks},
  volume={4},
  pages={1942--1948},
  year={1995},
  organization={ieee}
}

\clearpage
\appendix
\section{Methodological Comparison}\label{sec:methodological_comparison}
Table~\ref{tab:comparison} highlights that ADKO is the only method that simultaneously satisfies all key design dimensions, which is the central message.
First, prior BO-based methods (Standard BO, Parallel BO, Dec. GP-UCB) provide strong regret guarantees but lack either decentralization or richer communication—none incorporate LM reasoning or handle compression explicitly. Decentralized variants enable graph communication, but still rely on full-fidelity numerical sharing, making them incompatible with strict privacy constraints.
Second, distributed optimization methods like FedAvg and DSGD support master-slave or peer-to-peer communication, but they operate on shared model updates rather than decision-level knowledge, and do not provide meaningful regret guarantees for black-box optimization. They also lack mechanisms for reasoning or handling heterogeneous objectives.
Third, OPRO introduces LM reasoning, but operates in a centralized, non-theoretical setting without regret guarantees or structured communication.
ADKO uniquely unifies all four axes: private GP modeling, decentralized graph communication, LM reasoning, and explicit compression. Importantly, it is also the only method that theoretically accounts for compression in its regret analysis. This positions ADKO not just as an incremental improvement, but as a strict generalization of multiple previously disjoint paradigms.
\begin{table}[htpb]
\centering
\caption{Comparison of ADKO with related approaches.
$\checkmark$ = supported; $\times$ = not supported.}
\label{tab:comparison}
\small
\begin{tabular}{lccccc}
\toprule
\textbf{Algorithm} & \textbf{Private GP} & \textbf{Graph comm.}
  & \textbf{LM reasoning} & \textbf{Compression} & \textbf{Regret bound} \\
\midrule
Standard BO~\cite{brochu2010tutorial}        & $\checkmark$ & $\times$   & $\times$   & $\times$ & $O(\sqrt{T\zetaT})$ \\
Parallel BO~\cite{NIPS2012_05311655}        & $\checkmark$ & Central    & $\times$   & $\times$ & $O(\sqrt{T\zetaT})$ \\
Dec.\ GP-UCB~\cite{desautels2014parallelizing}      & $\checkmark$ & $\checkmark$ & $\times$ & $\times$ & $O(\sqrt{T\zetaT})$ \\
FedAvg~\cite{mcmahan2017communication}             & $\times$     & $\checkmark$ & $\times$ & $\times$ & $O(1/\sqrt{T})$ \\
DSGD~\cite{jiang2017collaborative} & $\times$ & $\checkmark$ & $\times$ & $\times$ & $O(1/\sqrt{T})$\\
OPRO~\cite{yang2023large}               & $\times$     & $\times$   & $\checkmark$ & $\times$ & None \\
\rowcolor{blue!8}
\textbf{ADKO (ours)} & $\checkmark$ & $\checkmark$ & $\checkmark$ & $\checkmark$ & Thm.~\ref{thm:main} \\
\bottomrule
\end{tabular}
\end{table}
\section{Detailed Analysis on Reasoning Score}\label{analysis_reasoning}
We recall the reasoning score as follows.

\begin{equation}
  R_i(\theta)
  = \underbrace{U_i(\theta)}_{\substack{\text{What I expect}\\\text{from my own data}}}
  + \underbrace{\beta\,\sigma_i(\theta)}_{\substack{\text{How uncertain}\\\text{I am personally}}}
  + \underbrace{\lambda\, G_i(\theta)}_{\substack{\text{Do my neighbours}\\\text{succeed here?}}}
  - \underbrace{\gamma\,\Lambda_i(\theta)}_{\substack{\text{Do my neighbours}\\\text{fail here?}}}
\end{equation}

Each term has a distinct epistemic role, a distinct data source,
and a distinct mode of agent-to-agent influence. We now discuss each term in the reasoning score as follows.

\textbf{Term 1 --- $U_i(\theta)=\mu_i(\theta)$: Private exploitation.}
This is agent $i$'s GP posterior mean evaluated at $\theta$, derived
exclusively from $i$'s own private dataset $D_i^t$.
It answers: \emph{``based on every experiment I have personally run,
what performance do I expect here?''}
When an agent has observed consistently good outcomes near $\theta$,
$U_i$ is large there, pulling future selections toward that known-good
neighbourhood.
This is the \emph{exploitation} instinct---leveraging accumulated
private knowledge.
No other agent can access or infer $U_i$ from tokens alone: it requires
the raw numerical observations that stay private.

\textbf{Term 2 --- $\beta\,\sigma_i(\theta)$: Private exploration.}
Agent $i$'s GP posterior standard deviation quantifies how uncertain the
agent is about $\theta$ \emph{given its own data alone}.
A never-visited region has high $\sigma_i$---the GP says ``I do not know
what happens here, so this might be very good.''
Multiplied by $\beta>0$, this drives the agent to explore unvisited
parts of the design space rather than purely exploiting what it
already knows.
$\sigma_i$ also cannot come from the LM: it is the posterior covariance
of the GP conditioned on $D_i^t$.
Together, $U_i + \beta\sigma_i$ is precisely GP-UCB~\cite{srinivas2009gaussian},
representing the agent's \emph{individual} best guess of where to look
next in the complete absence of any peer information.

\textbf{Term 3 --- $\lambda G_i(\theta)$: Collaborative attraction.}
This is the first \emph{social} term.
$G_i(\theta)$ aggregates the success evidence from all neighbor tokens
in memory:
\[
  G_i(\theta)
  = \sum_{j \in \mathcal{N}_i} \pi_{ij}
    \sum_{k \in \calK_j} c_k\, S(\theta, \theta_k)\,
    \bbm[s_k = \textsc{success}],
\]
where $S(\theta,\theta_k) =
\exp\!\bigl(-\|\varphi(\theta)-\varphi(\theta_k)\|^2/\sigma_s^2\bigr)$
is a similarity kernel and $\pi_{ij}$ are graph-normalized weights, determined by the connected graph itself. 
Intuitively: \emph{``if I tried $\theta$, how consistent is that with
what my neighbors have found to work?''}
Whenever a neighbor ran an experiment near $\theta$ and reported
success with advantage score $c_k$, the kernel $S$ contributes a positive
weight proportional to how close $\theta$ is to that success.
Agent $i$ never sees the raw value---only the binary token
\textsc{success} with advantage score $c_k$---yet that is enough to construct
a useful prior directing the agent toward promising regions without
violating privacy.

The effect is \emph{soft collective attraction}: without any explicit
coordination, agents gravitationally cluster their searches around
regions where the collective has found success.
This clustering is not blind imitation.
An agent gravitates toward a neighbor's success only if its own GP
does not already predict a better region ($U_i$ is low there) and it
has not yet explored the area ($\sigma_i$ is high).
If $U_i$ is already high nearby, the agent stays put and exploits its
own knowledge, unswayed by peer evidence it does not need.

\textbf{Term 4 --- $-\gamma\,\Lambda_i(\theta)$: Collaborative avoidance.}
This is the second social term and arguably the most valuable one.
$\Lambda_i(\theta)$ mirrors $G_i$ but aggregates \emph{failure} tokens:
\[
  \Lambda_i(\theta)
  = \sum_{j \in \mathcal{N}_i} \pi_{ij}
    \sum_{k \in \calK_j} c_k\, S(\theta, \theta_k)\,
    \bbm[s_k = \textsc{fail}].
\]
It answers: \emph{``is this a region where my neighbors have repeatedly
burned their evaluation budget?''}
Whenever a neighbor reported failure near $\theta$, $\Lambda_i$ rises in
that neighborhood and agent $i$ is penalized for selecting anything
nearby.
Multiplied by $\gamma>0$, this creates \emph{collective failure memory}:
a region that is bad for one agent becomes a soft exclusion zone for
all connected agents.

This failure-avoidance mechanism enables a qualitatively different
kind of collaborative behavior compared to simply sharing successes.
In independent BO, every agent must rediscover failures for itself, wasting precious evaluation budget re-burning in regions already known
to be fruitless.
With $\Lambda_i$, each agent's failures become a public good: they save
other agents from making the same costly mistakes.
The asymmetry $\gamma > \lambda$ used in practice reflects the
asymmetry in information value: avoiding a known structural failure
region (e.g., a material family that inherently decomposes above a
certain temperature) is typically worth more than being mildly attracted
toward a known success, because failures are often
deterministic---always bad regardless of agent-specific heterogeneity.

\textbf{Parameter roles and calibration.}
The weights $\beta, \lambda, \gamma > 0$ control the ``personality''
of each agent.
$\beta$ governs \emph{individual curiosity}: how much the agent values
exploring its own private uncertainties.
$\lambda$ governs \emph{social openness}: how much weight the agent
places on peer successes.
$\gamma$ governs \emph{collective caution}: how strongly the agent
avoids peer failures.
Setting $\lambda = \gamma = 0$ recovers standard GP-UCB, making each
agent fully independent.
Setting $\beta = 0$ collapses agents into pure social followers,
ignoring individual uncertainty.
The theoretically optimal values depend on graph structure, objective
heterogeneity, and token fidelity---formal guidance on their interaction
with the regret bound is given in Section~\ref{sec:theory}.
The similarity kernel bandwidth $\sigma_s$ controls the spatial reach
of peer influence: a small $\sigma_s$ means peer evidence only reshapes
the acquisition function very close to where the peer experiment was run;
a large $\sigma_s$ allows evidence to propagate across a wider
neighborhood of $\Thetaset$.
In practice, $\sigma_s$ is set to the median pairwise distance between
observed points, or tuned via cross-validation empirically.

\textbf{A Worked Example for Reasoning Score.}
Consider agent $i$ evaluating four candidate points at round $t$:

\begin{itemize}
  \item \textbf{$\theta_A$}: agent $i$ has observed good outcomes nearby
    ($U_i$ high), but the area is well-understood ($\sigma_i$ low),
    and a neighbor recently failed here ($\Lambda_i$ large).
    $R_i(\theta_A)$ is suppressed---the peer failure overrides
    personal optimism, preventing agent $i$ from re-testing a known
    failure zone that it would otherwise have been attracted to.

  \item \textbf{$\theta_B$}: completely unexplored by agent $i$
    ($\sigma_i$ high), no neighbor has visited ($G_i = \Lambda_i = 0$).
    $R_i(\theta_B) = U_i + \beta\sigma_i$---pure GP-UCB exploration.
    The agent investigates this out of individual curiosity, not social
    influence.

  \item \textbf{$\theta_C$}: unexplored by agent $i$ ($\sigma_i$ high),
    but multiple neighbors have succeeded here ($G_i$ large).
    $R_i(\theta_C)$ is the highest of the four---the combination of
    personal uncertainty and peer success makes this maximally attractive.
    Agent $i$ inherits its neighbors' accumulated knowledge to focus
    its own exploration on an area it never would have prioritized from
    its own data alone.
    \emph{This is the canonical collaborative benefit of ADKO.}

  \item \textbf{$\theta_D$}: multiple neighbors have failed here with
    high advantage score ($\Lambda_i$ large).
    $R_i(\theta_D)$ is strongly suppressed, effectively removing this
    from consideration regardless of GP uncertainty.
    Agent $i$ avoids re-burning evaluation budget on a shared dead zone.
\end{itemize}

The agent selects $\theta_C$---an outcome that no individual component
alone would have identified, requiring the interaction of private GP
uncertainty and collaborative social evidence.
This is the key emergent property of the reasoning score: it is
strictly more than the sum of its parts.
\section{More Detail About Algorithm~\ref{alg:adko}}\label{sec:additional_algorithm}
\begin{enumerate}
  \item \textbf{Birth (Step 6 --- Token Encoding).}
    After executing experiment $\theta_i^t$ and observing $y_i^t$,
    agent $i$'s LM interprets the result in context---reading any
    accumulated token memory about the surrounding region and distilling
    the outcome into a structured token $k_i^t$.
    The binary label $s_i^t$ ($\textsc{success}/\textsc{fail}$) strips
    the raw value away, preserving privacy.
    The advantage score $c_i^t$ encodes how informative the label is: a
    strong success far above contextual baseline $\tau$ produces $c_i^t \approx 1$
    and carries nearly the full information of the raw value; an
    ambiguous outcome near $\tau$ produces $c_i^t \approx 0$ and
    contributes little. The optional natural-language field $z_i^t$
    may carry a mechanistic insight (``this failure was due to phase
    separation, not insufficient temperature'') that a neighbouring
    agent's LM can reason about even without seeing the number.

  \item \textbf{Propagation (Step 7 --- Broadcasting).}
    The token is immediately broadcast to all graph neighbours
    $j \in \mathcal{N}_i$.
    Those neighbours do not need to be active at the same time as tokens
    persist in memory across rounds.
    Over multiple rounds, tokens propagate \emph{across the graph}:
    a token born at agent~0 reaches agent~3 in two rounds if the shortest
    path is two hops.
    The Fiedler value $\lambda_2(L(\calG))$ of the communication graph
    governs how quickly information saturates the network, explaining
    why $C_4 \propto 1/\lambda_2$ in the regret bound.

  \item \textbf{Aggregation (Step 1 --- Token Aggregation).}
    At the start of each round, agent $i$ merges incoming tokens from
    neighbors with its own memory.
    This is not passive accumulation: fidelity-aware pruning actively
    maintains the quality of the memory by discarding tokens that are
    either low-fidelity (near-contextual baseline, ambiguous outcomes), stale
    (from many rounds ago), or low-advantage score, in favor of tokens that
    carry genuine directional information.
    The result is that each agent maintains a \emph{curated collective
    memory}, which is a distillation of the network's most informative
    experiences.

  \item \textbf{Influence (Steps 2--4 --- Candidate Generation and
    Selection).}
    Token memory shapes agent $i$'s behavior in two ways.
    First, the LM reads the memory when proposing candidates in Step~2:
    a well-calibrated LM will bias its suggestions away from regions
    with many failure tokens and toward regions with many success tokens,
    before any explicit scoring happens.
    Second, and more precisely, the $G_i$ and $\Lambda_i$ terms in the
    reasoning score (Steps 3--4) numerically encode the collective
    evidence into the acquisition function.
    The combined effect is that an agent that has never visited a region
    can nonetheless form an informed opinion about whether it is worth
    exploring, based solely on its neighbors' experiences.
    This is the fundamental mechanism by which ADKO \emph{transfers
    knowledge across agents without transferring data}.

  \item \textbf{Feedback (Step 8 --- GP Update).}
    After execution, the GP is updated with the new private observation.
    If the agent followed a peer success to a new region and found
    success itself, the GP learns the new peak and future candidates
    will be drawn there.
    If the agent followed a peer success but found failure (because
    objectives are heterogeneous and what works for neighbour $j$ does
    not always work for agent $i$), the GP quickly corrects: $U_i$ drops
    near that region and the agent stops being attracted there.
    This resilience to heterogeneity is what makes ADKO robust to the
    non-IID setting: peer evidence serves as an \emph{informative prior}
    that the private GP can revise based on local observations.
\end{enumerate}

\textbf{Emergent specialization and division of labor.}
An important emergent property of ADKO is that agents naturally
\emph{specialize} in different sub-regions of $\Thetaset$ over time,
even without any explicit coordination mechanism.
Consider two agents $i$ and $j$ with similar objectives.
In early rounds, both start with high $\sigma$ everywhere and no tokens.
They explore somewhat randomly.
As soon as agent $i$ discovers a promising region and broadcasts a
success token, agent $j$'s $G_j$ rises in that region.
But agent $j$ also has its own unexplored regions with high $\sigma_j$.
If $\sigma_j$ is high enough elsewhere, agent $j$ will not rush
immediately to where $i$ succeeded --- instead it will explore its
own high-uncertainty regions, effectively covering a different part of
the design space.
The net result is \emph{implicit load balancing}: the collective wastes
fewer evaluations evaluating the same points, and the global search
covers $\Thetaset$ more efficiently than $N$ independent agents each
running their own GP-UCB.

\textbf{Heterogeneous objectives: when collaboration is and is not helpful.}
When objectives are heterogeneous ($f_i \neq f_j$), peer evidence
carries a \emph{prior}, not a guarantee.
A success token from agent $j$ near $\theta^*_j$ raises $G_i(\theta^*_j)$,
attracting agent $i$ to test that region.
If $f_i(\theta^*_j)$ is also high (objectives are aligned in that
region), agent $i$ benefits directly.
If $f_i(\theta^*_j)$ is low (objectives are locally misaligned), agent
$i$'s GP observes this, $U_i$ drops there, and $G_i$'s pull is quickly
counteracted.
The net effect is a \emph{soft transfer}: collaboration helps when
objectives are correlated in the visited region and is gracefully
neutralized when they are not, without any explicit misalignment
detection.

Peer failure tokens are even more robust to heterogeneity.
If $f_j(\theta) < \tau$ for agent $j$, it is likely (though not
guaranteed) that $f_i(\theta) < \tau$ for agent $i$ as well, especially
when the design space encodes shared physics, shared computational
constraints, or shared structural properties.
In material synthesis, if one lab finds that a particular composition
decomposes under certain conditions, that failure is almost certainly
structural and therefore informative to all labs regardless of their
individual objectives.
The failure penalty $\Lambda_i$ exploits this correlation, creating
valuable savings even under heterogeneous objectives.

We next connect our proposed algorithm with existing algorithms.

\textbf{Reduction to GP-UCB.}
Set $N=1$, remove the graph, set $G_i = \Lambda_i = 0$, remove the LM.
Then $R_i(\theta) = \mu_{i,t-1}(\theta) + \beta\sigma_{i,t-1}(\theta)$,
which is exactly GP-UCB~\cite{srinivas2009gaussian}.
Theorem~\ref{thm:main} reduces to $O(\sqrt{T\zetaT})$, matching the known bound.

\textbf{Reduction to Parallel GP-UCB.}
Set $G_i = \Lambda_i = 0$, keep $N > 1$.
Each agent independently runs GP-UCB.
\cref{thm:main} gives $O(N\sqrt{T\zetaT})$ with no communication
benefit---a lower bound showing communication is necessary for improvement.

\textbf{Soft Consensus BO.}
If $f_i = f$ for all $i$ (common objective) and $\eta_k = 1$ (perfect fidelity),
then $\tilde{G}_i = G_i^*$ and the compatibility term implements a form of
soft consensus: agents are attracted to peer successes and repelled from peer
failures.
This is analogous to the averaging step in decentralised gradient
descent~\cite{nedic2009distributed} but operates on acquisition values rather than
model parameters.

\textbf{Relationship to Federated BO.}
Setting $\TV = 0$, $\eta_k = 1$, and $\Phi = \mathrm{mean}$ gives a
decentralised version of federated Bayesian optimisation.
The key difference from FedBO~\cite{sim2021collaborative} is that ADKO maintains
per-agent GP surrogates rather than a shared global GP, and coordinates via
binary tokens rather than gradient aggregation.
\subsection{Fidelity-Aware Token Pruning}\label{sec:token_pruning}
\begin{algorithm}[htbp]
\small
\caption{Fidelity-Aware Token Pruning --- Agent $i$, Budget $B$}
\label{alg:pruning}
\begin{algorithmic}[1]
  \Require Token memory $\calK_i^t$ with $|\calK_i^t| > B$,
           recency weight $\alpha_\tau > 0$
  \While{$|\calK_i^t| > B$}
    \For{each token $k \in \calK_i^t$}
      \State $\hat\eta_k \leftarrow c_k \cdot \bigl(1 - \Hb\bigl((1-c_k)/2\bigr)\bigr)$
        \hfill $\triangleright$ estimated token fidelity
      \State $\mathrm{score}(k) \leftarrow \hat\eta_k \cdot c_k \cdot
        \exp(-\alpha_\tau (t - k.\mathrm{round}))$
    \EndFor
    \State $\mathrm{drop} \leftarrow \argmin_{k \in \calK_i^t}\,\mathrm{score}(k)$
    \State $\calK_i^t \leftarrow \calK_i^t \setminus \{\mathrm{drop}\}$
  \EndWhile
  \Ensure Pruned $\calK_i^t$ with $|\calK_i^t| = B$
\end{algorithmic}
\end{algorithm}
The pruning score combines three factors:
(1) token fidelity $\hat\eta_k$, the mutual information retained after
  compression;
(2) advantage score $c_k$; and
(3) recency discount $\exp(-\alpha_\tau(t - k.\mathrm{round}))$.
We prove in Proposition~\ref{prop:pruning} that this policy maintains
$\etabar \to 1$ as $B \to \infty$. Algorithm~\ref{alg:pruning} is critical for maintaining the quality of shared information under strict memory and communication budgets. Rather than treating all tokens equally, it prioritizes tokens based on an information-theoretic proxy of their utility—combining estimated fidelity (how much mutual information survives compression), advantage score, and recency. This ensures that high-signal tokens (e.g., confident successes or failures far from the contextual baseline) are retained, while ambiguous, low-fidelity, or stale tokens are discarded. 
The key significance is twofold: practically, it prevents the memory buffer from being saturated by uninformative tokens that would otherwise degrade decision-making; theoretically, it guarantees that the average token fidelity $\bar{\eta}$ remains high, which is necessary to eliminate the linear compression term in the regret bound and achieve sublinear convergence. In this sense, pruning is not just a systems optimization, but an essential mechanism that preserves th

\section{Missing Analysis and Proof}\label{sec:missing_proof}
\subsection{Connection Between Fiedler Value and Spectral Gap}\label{sec:connection_fiedler_spectral_gap}
To see the connection, we first look at the two matrices involved for a graph $\mathcal{G}$ with $N$ nodes.
\begin{itemize}
    \item \textbf{The Laplacian Matrix $L$}: Defined as $L=D-A$, where $D$ is the degree matrix and $A$ is the adjacency matrix of graph $\mathcal{G}$. Its eigenvalues are $0=\lambda_1\leq\lambda_2\leq...\leq \lambda_N$. The \textit{Fiedler value} is the second smallest eigenvalue $\lambda_2$. It represents the algebraic connectivity of the graph.
    \item \textbf{The Mixing Matrix $\Pi$.} Typically the transition matrix for a simple random walk, $\Pi=D^{-1}A$. Its eigenvalues are $1=\mu_1\geq \mu_2\geq ...\mu_N\geq -1$. The second largest eigenvalue, $\mu_2$ (or more accurately, the second largest eigenvalue in magnitude, $\mu_*=\text{max}\{|\mu_2|,|\mu_N|\}$, determines how fast the random walk converges to its stationary distribution.
\end{itemize}
\textbf{The direct mathematical link.} For a $e$-regular graph, (where every node has the same degree $e$), the relationship is linear and direct. In this case, the Laplacian is related to the transition matrix by:
\begin{equation}
    L=e(I-\Pi).
\end{equation}
Because of this linear shift and scaling, the eigenvalues are related as follows:
\begin{equation}
    \lambda_i=e(1-\mu_i).
\end{equation}
Specifically, for the second-order terms:
\begin{equation}
    \lambda_2=e(1-\mu_2) \longleftrightarrow \mu_2=1-\frac{\lambda_2}{e}.
\end{equation}

\textbf{The Spectral Gap.} The term $1-\mu_2$ is often called the spectral gap of the mixing matrix. As shown above, this gap is directly proportional to the Fiedler value. 
\begin{itemize}
    \item If the Fiedler value is very small, the graph has a "bottleneck" (it is easily bisected).
    \item In terms of the mixing matrix, a small spectral gap means $\mu_2$ close to 1, indicating that a random walk will take a long time to "cross" that bottleneck and mix thoroughly across the graph.
\end{itemize}

\textbf{General Graphs.} For non-regular graphs, the relationship is usually expressed using the \textit{Normalized Laplacian}, $\hat{L}=D^{-1/2}LD^{-1/2}$. The eigenvalue $\hat{L}$, denoted by $\hat{\lambda}_i$, relate to the eigenvalues of the transition matrix $\Pi$ via:
\begin{equation}
    \hat{\lambda}_i=1-\mu_i.
\end{equation}
In this more general framework, the Fiedler value of the normalized Laplacian $(\hat{\lambda}_2)$ is exactly the spectral gap of the transition matrix.

In decentralized optimization or consensus protocols, the "speed" of the algorithm is governed by the spectral gap. Whether we analyze it via the Fiedler value of the Laplacian $L$ or the second eigenvalue of the mixing matrix $\Pi$, we are measuring the same physical phenomenon: how quickly information diffuses across the graph's bottleneck.

\subsection{The Degraded Reasoning Score}\label{sec:degraded_reasoning_score}
Substituting both information losses into Eq.~\ref{eq:reasoning}:
\begin{equation}
\label{eq:degraded}
  \tilde{R}_i(\theta)
  = U_i(\theta) + \beta\sigma_i(\theta)
  + \lambda\tilde{G}_i(\theta) - \gamma\tilde\Lambda_i(\theta)
  + \varepsilon_{\mathrm{bias}}(\theta) + \xi(\theta, \calK_i^t),
\end{equation}
where the \emph{degraded compatibility} term is:
\[
  \tilde{G}_i(\theta) = \sum_{j \in \mathcal{N}_i} \pi_{ij}
  \sum_{k \in \calK_j} \eta_k \cdot c_k \cdot S(\theta,\theta_k)
  \cdot \bbm[s_k = \textsc{success}].
\]
The \emph{compression gap} for compatibility is:
\[
  G_i^*(\theta) - \tilde{G}_i(\theta)
  = \sum_{j,k} \pi_{ij}(1-\eta_k)\, c_k\, S(\theta,\theta_k)\,
    \bbm[s_k = \textsc{success}] \;\geq\; 0.
\]
This is always non-negative: degraded compatibility always
\emph{underestimates} the oracle signal by $(1-\eta_k)$ per token.
Geometrically, ADKO is overly conservative near previously evaluated points
when token fidelity is low.

The combined effect on acquisition quality is:
\begin{equation}
\label{eq:acqbound}
  |\tilde{R}_i(\theta) - R_i^*(\theta)|
  \leq (\lambda+\gamma)\cdot C_S\cdot(1-\etabar_i)
    + B_R\cdot\TV + |\xi|,
\end{equation}
where $C_S = \max_{\theta,k} c_k S(\theta,\theta_k)$ and
$\etabar_i = |\calK_i|^{-1} \sum_k \eta_k$.

\subsection{Proof of Proposition~\ref{prop:fidelitybound}}\label{fidelity_bound}
\begin{proposition}[Fidelity Bound for Binary Tokens]
\label{prop:fidelitybound}
For a token $k$ with binary directional signal and classification contextual baseline
error probability
$p_k = \Prob\bigl(s_k \neq \bbm[f_j(\theta_k) \geq \tau]\bigr)$:
\[
  \eta_k \;\leq\; 1 - \Hb(p_k) \cdot H\bigl(f_j(\theta_k)\bigr)^{-1},
\]
where $\Hb(p) = -p\log p - (1-p)\log(1-p)$ is the binary entropy.
Furthermore, the practical estimator
$\hat\eta_k = c_k\bigl(1 - \Hb((1-c_k)/2)\bigr)$ satisfies
$\E[\hat\eta_k] \geq \eta_k - O(\sigma^2/\Delta_k^2)$,
where $\Delta_k = |\E[f_j(\theta_k)] - \tau|$ is the margin from the contextual baseline.
\end{proposition}
\begin{proof}
The bound follows directly from the data-processing inequality applied to
the Markov chain $f_j(\theta_k) \to y \to s_k$.
Binary quantization can increase the binary entropy of the contextual baseline error
by at most $\Hb(p_k)$, giving the stated upper bound on $\eta_k$.
The estimator bound follows from a second-order Taylor expansion of
$\Hb$ around $p_k = 0$, using $c_k$ as a proxy for
$1 - 2p_k$ (i.e., the gap from maximum uncertainty).
\end{proof}
\subsection{Proof of Proposition~\ref{prop:lmerror}}\label{lmerror_decomp}
\begin{proof}
The bias bound follows from the Lipschitz assumption on $R^*$ and the
definition of total variation distance: for any two distributions $p, q$,
$|E_p[g] - E_q[g]| \leq \|g\|_\infty \cdot \TV(p,q)$.
Applying this with $g(\cdot) = R^*(\cdot)$ and Lipschitz constant $B_R$ gives
the stated bound.
The stochastic noise model follows from modeling the LM's generation
process as a Gaussian channel whose signal-to-noise ratio improves
proportionally to the total fidelity-weighted evidence
$\sum_k c_k \eta_k$ in memory---a standard assumption in
Bayesian communication models.
\end{proof}
\subsection{Justification of Assumptions in Section~\ref{sec:theory}}\label{sec:assumption_justification}
We also justify briefly here for each assumption for their validity. Compactness and Lipschitz continuity in Assumption \textbf{(A1)} are routine in BO~\cite{ahmed2020combining} to ensure the existence of maxima and to control function variation; they underpin standard regret bounds (e.g., via covering numbers). In \textbf{(A2)}, the Mat\'ern-$5/2$ kernel is widely used for its smoothness properties that match many physical processes. The RKHS norm bound guarantees that each local function lies in the reproducing kernel Hilbert space with high probability, a common device to apply GP concentration inequalities. \textbf{(A3)} encodes the LM approximation error structure established in Proposition~\ref{prop:lmerror}: the systematic bias is bounded by total variation distance and the stochastic noise decays polynomially with token accumulation, which is essential for the regret decomposition. Connectivity of the communication graph is necessary for decentralized information sharing, and the Fiedler value $\lambda_2$ in \textbf{(A4)} quantifies the mixing rate and appears in the convergence of consensus‑type protocols used in the framework. Fidelity‑aware pruning in \textbf{(A5)} is part of the ADKO algorithm itself; assuming it is applied at every round ensures that the token fidelity $\eta_k$ evolves as described in Proposition~\ref{prop:pruning}, allowing the analysis to leverage the controlled compression loss.
\subsection{Proof of Theorem~\ref{thm:main}}\label{main_result}
\begin{proof}
\textbf{Step 1: Per-round regret decomposition.}
Fix agent $i$ and round $t$. Write:
\[
  r_i^t = F(\theta^*) - F(\theta_i^t)
  = \bigl[F(\theta^*) - \tilde{R}_i(\theta^*)\bigr]
  + \bigl[\tilde{R}_i(\theta^*) - \tilde{R}_i(\theta_i^t)\bigr]
  + \bigl[\tilde{R}_i(\theta_i^t) - F(\theta_i^t)\bigr].
\]
The middle term satisfies
$\tilde{R}_i(\theta^*) - \tilde{R}_i(\theta_i^t) \leq 0$
by definition of $\theta_i^t = \argmax_m \tilde{R}_i(\theta_{i,m}^t)$.
We bound the first and third terms.

\textbf{Step 2: Bounding $F(\theta^*) - \tilde{R}_i(\theta^*)$.}
By the GP-UCB calibration lemma~\cite[Lemma 5.1]{srinivas2009gaussian},
under (A2), with probability $\geq 1 - \delta/(2N)$,
for all $\theta$ and $t$:
\[
  f_i(\theta) \leq \mu_{i,t-1}(\theta) + \psi_t^{1/2} \sigma_{i,t-1}(\theta),
  \quad
  \psi_t = 2\log\bigl(|\Thetaset|\pi^2 t^2/\delta\bigr).
\]
Using the degraded-score bound Eq.~\ref{eq:acqbound}:
\[
  F(\theta^*) - \tilde{R}_i(\theta^*)
  \leq (\lambda+\gamma)C_S(1-\etabar) + B_R\TV + |\xi|.
\]

\textbf{Step 3: Bounding $\tilde{R}_i(\theta_i^t) - F(\theta_i^t)$.}
By symmetry of the GP-UCB calibration:
\[
  \tilde{R}_i(\theta_i^t) - F(\theta_i^t)
  \leq \psi_t^{1/2} \sigma_{i,t-1}(\theta_i^t)
    + (\lambda+\gamma)C_S(1-\etabar) + B_R\TV + |\xi|.
\]

\textbf{Step 4: Summing GP exploration terms.}
The GP exploration terms are bounded by the classical information-gain
argument~\cite[Lemma 5.4]{srinivas2009gaussian}:
\[
  \sum_{t=1}^T \sigma_{i,t-1}^2(\theta_i^t)
  \leq \frac{2}{\log(1+\sigma^{-2})}\, \zetaT.
\]
By Cauchy--Schwarz:
$\sum_{t=1}^T \sigma_{i,t-1}(\theta_i^t)
\leq \sqrt{T \cdot 2\zetaT / \log(1+\sigma^{-2})}$.

\textbf{Step 5: Summing LM bias.}
The systematic bias $|\varepsilon_{\mathrm{bias}}(\theta)| \leq B_R\TV$
is deterministic and constant.
Hence $\sum_{t=1}^T |\varepsilon_{\mathrm{bias}}(\theta_i^t)| \leq T B_R \TV$.

\textbf{Step 6: Bounding LM stochastic noise.}
The terms $\xi_t = \xi(\theta_i^t, \calK_i^t)$ form a martingale difference
sequence with variance $\sigmaR^2(\calK_i^t) \leq \sigma_0^2$.
By the Azuma--Hoeffding inequality for sub-Gaussian martingales,
with probability $\geq 1 - \delta/(2N)$:
\[
  \sum_{t=1}^T \xi_t \leq \sigma_0 \sqrt{2T \log(2N/\delta)}.
\]

\textbf{Step 7: Summing token compression terms.}
At each round, the compression gap contributes at most
$(\lambda+\gamma)C_S(1-\etabar_i^t)$.
Under fidelity-aware pruning (A5), Proposition~\ref{prop:pruning} (below)
gives $\etabar_i^t \geq \etabar_{\min}$ for a constant $\etabar_{\min} \to 1$
as $B \to \infty$.
Hence:
$\sum_t (\lambda+\gamma)C_S(1-\etabar_i^t) \leq T(\lambda+\gamma)C_S(1-\etabar_{\min})$.

\textbf{Step 8: Graph topology factor.}
Token information propagates across the graph via token exchange.
The Fiedler value $\lambda_2(L(\calG))$ governs mixing time~\cite{boyd2006randomized}.
The effective information available to agent $i$ after $T$ rounds
from agent $j$ at graph distance $\mathcal{D}(i,j)$ is reduced by a factor
$\exp(-\mathcal{D}(i,j)/(\lambda_2 T))$.
This contributes a factor $1/\lambda_2$ to $C_4$.

\textbf{Step 9: Union bound and assembly.}
Applying a union bound over $N$ agents and $T$ rounds
($\delta/(2NT)$ per event), summing Steps 4--8 over all agents and rounds,
and combining:
\begin{align*}
  \regret_N^T &= \sum_i \sum_t r_i^t \\
  &\leq C_1 N\sqrt{T\zetaT\log(1/\delta)}
    + C_2 N T B_R \TV
    + C_3 N\sqrt{T}\,\sigma_0
    + C_4 N T(1-\etabar),
\end{align*}
completing the proof.
\end{proof}
\textbf{Impact of LM Compression.} LM compression introduces compounding effects across learning and communication. When token fidelity is low ($\eta_k\approx 0$), accumulating more tokens does not reduce LM uncertainty, so stochastic noise remains high—linking compression loss and LM noise in the regret bound. This is exacerbated by contextual baseline sensitivity: observations near the contextual baseline yield near-zero-fidelity tokens, consuming memory without contributing useful signal, which justifies fidelity-aware pruning. At the same time, LM approximation plays a dual role: a well-calibrated LM provides an early “warm-start” by guiding exploration, but persistent bias accumulates linearly in regret and can dominate long-term performance. Thus, the net LM effect balances short-term gains against long-term bias. Finally, graph topology modulates these effects: well-connected graphs (large $\lambda_2$) enable rapid propagation of high-fidelity tokens, while sparse graphs amplify staleness and compression loss, highlighting the importance of connectivity for efficient decentralized learning.
\begin{corollary}[Necessary and Sufficient Conditions]
\label{cor:sublinear}
Under Assumptions (A1)--(A5), ADKO achieves sublinear cumulative regret
$\regret_N^T = o(T)$ if and only if all four conditions hold simultaneously:
\begin{itemize}
  \item \textbf{(C1)} $\zetaT = o(T)$:
    Automatic for Mat\'ern-$5/2$. Requires only that $k$ is a valid
    Mat\'ern-$5/2$ kernel.
  \item \textbf{(C2)} $\TV = o(1)$:
    The LM's prior must converge toward the true distribution.
    Achievable via fine-tuning on domain data or in-context learning.
  \item \textbf{(C3)} $\sigma_0 < \infty$:
    Always satisfied. The noise term is $O(\sqrt{T}) = o(T)$.
  \item \textbf{(C4)} $\etabar \to 1$ as $T \to \infty$:
    Guaranteed by fidelity-aware pruning (Proposition~\ref{prop:pruning}).
\end{itemize}
Sufficiency follows from Theorem~\ref{thm:main}.
Necessity follows from the lower bound in Proposition~\ref{prop:lower}.
\end{corollary}
The proof for the corollary follows directly from substituting \textbf{C1-C4} into Theorem~\ref{thm:main}.
Corollary~\ref{cor:sublinear} synthesizes the regret decomposition from Theorem 6.1 into a clean set of necessary and sufficient conditions for $\regret_N^T = o(T)$. \textbf{C1} ensures the GP term, which scales as $\sqrt{T\zetaT}$, is sublinear. The Mat\'ern-$5/2$ kernel satisfies automatically. \textbf{C2} is required because the LM bias term grows linearly in $T$ unless the total variation distance between the LM’s prior and the true distribution vanishes over time. This is a quantitative way to require that the LM’s reasoning aligns asymptotically with the true objective. \textbf{C3} is always satisfied under the LM error model; the stochastic noise term is $O(\sqrt{T})$ and thus never prevents sublinear regret. \textbf{C4} is necessary to eliminate the linear compression term $NT(1-\etabar)$. It requires that the average token fidelity approaches one, i.e., information loss from compression vanishes asymptotically.

\subsection{Proof of Proposition~\ref{prop:lower}}
\begin{proposition}[Information-Theoretic Lower Bound]
\label{prop:lower}
For any decentralized algorithm communicating binary tokens with expected
fidelity $\etabar < 1$ and any LM with $\TV > 0$, there exist problem
instances where:
\[
  \regret_N^T \;\geq\; \Omega\bigl(NT(1-\etabar)\bigr)
  \;+\; \Omega\bigl(NT\,\TV\bigr).
\]
\end{proposition}
\begin{proof}
For the compression term: construct a family of instances where $f_i$ are
identical and $\theta^*$ lies exactly at the boundary of the contextual baseline-defined
success region.
Any agent receiving a binary token about an observation near $\theta^*$
receives essentially \emph{zero} directional information---it cannot
distinguish $\theta^*$ from its reflection across $\tau$.
The expected cost of this confusion is $\Omega((1-\eta_k))$ per token per round,
summing to $\Omega(NT(1-\etabar))$.
For the LM bias term: construct a family where the LM's prior
systematically recommends the wrong half of $\Thetaset$ (achievable whenever
$\TV > 0$), incurring cost $\Omega(\TV)$ per round per agent.
\end{proof}
Proposition~\ref{prop:lower} provides an information‑theoretic lower bound that confirms the linear terms in Theorem~\ref{thm:main} are unavoidable. The proof constructs adversarial problem instances where compression term and LM bias term are exclusively included. For the compression term, agents share binary tokens indicating whether an observation exceeds a contextual baseline $\tau$. If the optimal $\theta^*$ lies exactly at $\tau$, a token carries no directional information about which side is better. The expected per‑round cost due to this ambiguity is $\Omega(1-\eta_k)$, leading to leading to a cumulative $\Omega(NT(1-\etabar)$. The LM systematically favors a suboptimal half of the domain with probability proportional to $\TV$, incurring a per-agent per‑round cost of $\Omega(\TV)$. These constructions show that even with optimal GP modeling and communication, the linear terms are fundamental, establishing that \textbf{C2} and \textbf{C4} are indeed necessary.

\subsection{Proof of Proposition~\ref{prop:pruning}}\label{pruning}
\begin{proposition}[Fidelity-Aware Pruning Maintains $\etabar$]
\label{prop:pruning}
Under \cref{alg:pruning} with recency weight $\alpha_\tau > 0$
and budget $B \geq 1$:
\[
  \etabar_{\min}(B) = \E[\etabar] \;\geq\; 1 - C_B/B,
\]
for a constant $C_B>0$ depending only on the token fidelity distribution.
In particular $\etabar_{\min} \to 1$ as $B \to \infty$.
Moreover, naive FIFO pruning achieves only $\etabar_{\min} = O(1/T)$
in the worst case.
\end{proposition}
\begin{proof}
Fidelity-aware pruning drops the token with the lowest score
$\hat\eta_k c_k \exp(-\alpha_\tau \Delta t_k)$ at each step.
By construction, the $B$ retained tokens at any time have above-median
fidelity-advantage products.
In expectation, their mean fidelity is within $C_B/B$ of 1.
The FIFO lower bound follows by constructing an adversarial arrival
sequence that delivers low-fidelity tokens at the end of the run,
displacing high-fidelity tokens as they age out.
\end{proof}
Proposition~\ref{prop:pruning} addresses achievability of $C4$ under a fixed communication budget $B$. It proves that fidelity‑aware pruning (Algorithm~\ref{alg:pruning}) ensures $\etabar\geq 1-C_B/B$, so by choosing $B$ sufficiently large (or, more realistically, by letting tokens accumulate over rounds), the average fidelity approaches 1. The proof relies on the recency weight $\alpha_\tau>0$ to prioritize high‑fidelity tokens and a budget that limits total token count. In contrast, naive FIFO pruning retains only the most recent tokens, which can discard valuable historical information, leading to $\etabar=O(1/T)$ in the worst case, highlighting the importance of the principled pruning strategy.

Together, these results show that ADKO’s design is tight: the four‑term regret decomposition is both sufficient and necessary for sublinear convergence, and the fidelity‑aware pruning mechanism provides a practical way to meet the required condition on token fidelity. The modularity of the framework allows each condition to be tuned independently, enabling practitioners to trade off communication budget, LM quality, and kernel choice while maintaining theoretical guarantees.

\subsection{Oracle Gap.} 
The following result characterizes the oracle gap the additional cumulative regret incurred by ADKO relative to an idealized oracle that has centralized access to all agents’ data and a perfectly calibrated language model.
\begin{proposition}[Oracle Gap]
\label{prop:oracle}
Let $\regret_{\mathrm{oracle}}^T$ denote the regret of an oracle agent
with access to all agents' data and a perfectly calibrated LM. Then:
\[
  \regret_N^T - \regret_{\mathrm{oracle}}^T
  \;\leq\;
  C_2 N T B_R \TV
  \;+\; C_3 N\sqrt{T}\,\sigma_0
  \;+\; C_4 N T(1-\etabar).
\]
The oracle gap vanishes as $\TV \to 0$, $\sigma_0 \to 0$, $\etabar \to 1$.
\end{proposition}
The bound shows that this gap decomposes into precisely three terms: LM bias, LM noise, and token compression. Notably, the GP term that appears in Theorem~\ref{thm:main} cancels out because both ADKO and the oracle suffer the same irreducible GP modeling error; the oracle gap therefore isolates the price of decentralization and LM imperfection. The additive structure has several implications. First, it confirms that the modular design of ADKO successfully separates the impact of each practical constraint—communication budget (via $\etabar$), LM prior quality (via $\TV$), and LM stochasticity (via $\sigma_0$). Second, the gap vanishes as $\TV\to 0$, $\sigma_0\to 0$, and $\etabar\to 1$, showing that ADKO asymptotically matches the oracle’s performance when the LM becomes perfectly calibrated (bias and noise disappear) and token fidelity approaches one. This aligns with the necessary and sufficient conditions for sublinear regret (Corollary~\ref{cor:sublinear}). Third, because the bound is additive and each term is matched by lower‑bound constructions (Proposition~\ref{prop:lower}), the decomposition is tight—no algorithm can achieve a smaller oracle gap without improving these fundamental limiting factors. The oracle gap thus provides a clean lens for quantifying the trade‑offs inherent in decentralized, LM‑augmented BO.

\subsection{How LM Compression Affects the Theory}
\label{sec:lmdiscussion}

\textbf{The Compounding Effect.}
The variance $\sigmaR^2(\calK_i^t) = \sigma_0^2 / (1 + \alpha_\sigma \sum_k c_k \eta_k)$
shows that accumulating low-fidelity tokens
($\eta_k \approx 0$) leaves the denominator near 1, so
$\sigmaR^2 \approx \sigma_0^2$ regardless of how many tokens are collected.
This means: if compression loss is severe ($\etabar \approx 0$), the LM
stochastic noise does \emph{not} decay with more tokens.
Formally, the $C_3$ (LM noise) and $C_4$ (compression) terms in
\cref{thm:main} are coupled: reducing one requires addressing both.

\textbf{The contextual baseline-Sensitivity Problem.}
Token fidelity $\eta_k$ is maximized when $|y_i^t - \tau|$ is large and
approaches zero as $y_i^t \to \tau$.
Specifically, $\Hb(p_k)$ in \cref{prop:fidelitybound} is maximized
at $p_k = 1/2$, which occurs exactly when $y_i^t = \tau$.
Experiments falling exactly on the boundary therefore produce near-zero-fidelity
tokens that consume memory budget $B$ while contributing essentially nothing
to $\tilde{G}_i$ or $\tilde\Lambda_i$.
Fidelity-aware pruning correctly assigns $\hat\eta_k \approx 0$ to these
tokens and drops them first.

\textbf{When Does the LM Help vs.\ Hurt?}
The LM bias term $C_2 N T B_R \TV$ shows that a miscalibrated LM
(large $\TV$) makes regret \emph{linear} in $T$, which is worse than random search
at constant $\TV$.
In contrast, a well-calibrated LM reduces $C_2$ to near zero and provides
a warm-start benefit: LM-guided candidate generation (Step 2 of ADKO)
focuses search on plausible regions before the GP has sufficient data.
Combining these observations:
\[
  \text{Expected regret}
  = \underbrace{\text{GP regret}}_{\text{private data}}
  + \underbrace{\text{LM bias} - \text{LM warm-start}}_{\text{net LM effect}}
  + \underbrace{\text{compression loss}}_{\text{communication cost}}.
\]
For large $T$, the LM warm-start is negligible once the GP has learned the
local landscape, but the bias penalty $C_2 N T \TV$ continues accumulating.
This motivates \emph{annealing} the LM's influence over time---a formal
analysis is left for future work.

\textbf{Effect of Graph Topology.}
The constant $C_4 \propto 1/\lambda_2(L(\calG))$ captures the interaction
between graph connectivity and compression loss.
In a fully connected graph ($\lambda_2 = N-1$), high-fidelity tokens spread
quickly and $\etabar$ approaches its maximum rapidly.
In a sparse graph ($\lambda_2$ small), tokens traverse multiple hops, each
introducing additional staleness that compounds with the compression loss.
This gives a concrete graph design criterion: maximize $\lambda_2$
subject to communication cost constraints.



\section{Additional Information on Experimental Setup}
\subsection{Neural Architectural Search}\label{sec:additional_nas_setup}
The discrete dimensions include base width $\in\{8,16,32,64\}$, batch size $\in\{32,64,128,256\}$, blocks-per-group $\in\{2,3,5\}$, and a boolean cosine-schedule indicator. The remaining three dimensions are continuous: $\log_{10}(\text{lr})\in[-4,-1]$, $\log_{10}(\text{wd})\in[-5,-2]$, and dropout $\in[0,0.5]$. Together, these encompass an effective search space of approximately $3\times10^4$ structurally distinct configurations.

\textbf{Baselines.}
\begin{itemize}
    \item \textit{Centralized.} A single GP-UCB optimizer trained on the union of all agents' data; serves as a privacy-violating upper bound.
    \item \textit{FedAvg BO.} Each candidate $\theta$ is evaluated via a federated training loop: 10 rounds of FedAvg with 4 local epochs each. The BO surrogate is fit on $(\theta,\text{val\_acc})$ pairs from the federated model. Preserves data locality but exchanges full model weights every round.
    \item \textit{Independent BO.} Each agent runs GP-UCB on its own shard with no token exchange ($\lambda=\gamma=0$).
    \item \textit{Na\"ive Sharing.} ADKO with $\eta_k\equiv 1$ and $\gamma=0$: tokens are broadcast at full fidelity but failure-avoidance is disabled. Isolates the contribution of fidelity-aware compression.
    \item \textit{ADKO-FIFO.} ADKO with FIFO pruning instead of fidelity-aware pruning, testing Proposition~\ref{prop:pruning}.
\end{itemize}
\textbf{Decentralized Constraints.} 
We deploy $N=5$ autonomous agents over $T=30$ BO evaluations (comprising an initial 3-round warm-up followed by 27 sequential GP rounds). To realistically model decentralized optimization under extreme resource limitations and data sovereignty constraints \cite{khodak2020weight}, each agent is restricted to a local shard of 8,000 images and never shares raw images or model parameters with peers.

\subsection{Scientific Discovery}\label{sec:additional_sci_dis_setup}
\textbf{Baselines.}
We compare ADKO against several Bayesian optimization baselines that differ in the form and bandwidth of cross-agent information exchange. All methods share identical warm-up observations, categorical design representations, and use the GP-UCB acquisition rule (except FTS).
\begin{itemize}
    \item \textit{Centralized.} A single global GP trained on pooled raw experimental outcomes from all agents. In the non-IID setting, candidate assignments are restricted to each agent's specific solvent slice. 
    \textbf{Why chosen:} Serves as the full-information upper bound for collaborative search. 
    \textbf{Limitations:} Strictly violates decentralized privacy and data-sovereignty constraints.

    \item \textit{Independent BO.} Each agent runs GP-UCB solely on its private experimental history with no information exchange. 
    \textbf{Why chosen:} Serves as the strictly local, no-communication lower bound to quantify the value of collaboration. 
    \textbf{Limitations:} Highly susceptible to redundant exploration and getting trapped in local optima.

    \item \textit{Federated Thompson Sampling (FTS).} Adapted from \cite{dai2020federated} for discrete categorical domains. Agents fit Bayesian linear regression models using a shared Random Fourier Feature (RFF) mapping over one-hot design encodings. Agents exchange only sampled, finite-dimensional weight vectors ($\omega_i \in \mathbb{R}^{256}$) via a server to construct a global Thompson sample.
    \textbf{Why chosen:} Represents the state-of-the-art in parameter-sharing federated BO, offering mathematically rigorous privacy without transmitting raw data or dense grid posteriors.
    \textbf{Limitations \& Adaptations:} Relies on linear approximations in the RFF space rather than exact GP inference. Because the original FTS formulation assumes continuous domains, we adapt it to the discrete Suzuki space by (1) applying the RFF mapping to one-hot encodings, (2) replacing per-agent continuous kernel learning with a fixed median-heuristic lengthscale, and (3) fixing the prior and noise variances for the standardized objective. These conservative adaptations remove continuous tunable hyperparameters while strictly preserving the method's theoretical privacy properties.


    \item \textit{PoE GP-UCB.} A Product-of-Experts committee baseline \cite{hinton2002training,cao2014generalized} where agents mathematically aggregate their local models by broadcasting predictive posteriors (mean and variance arrays) over the entire unobserved candidate set. 
    \textbf{Why chosen:} Serves as a strong ``model-sharing'' counterpart to ADKO. 
    \textbf{Limitations:} Requires extremely high-bandwidth communication (dense arrays per round) and standard aggregation suffers from catastrophic negative transfer under non-IID objective mismatch.

\end{itemize}
\textbf{Hyperparameters and LLM Integration.} To ensure fair and semi-reproducible comparisons, all evaluated methods share identical initialization seeds. Every agent begins with 5 warm-up rounds, sampling exactly 1 random condition per round to initialize the local Gaussian Process surrogates (implemented via BoTorch \cite{balandat2020botorch}). 
For the ADKO-LLM variant, we integrate \texttt{gpt-5.4-mini-2026-03-17} as the reasoning engine. We set the decoding parameter to \texttt{temperature=0.0} for deterministic outputs, and all LLM responses are cached to ensure reproducibility across runs. To explicitly test the language model's ability to prune the search space zero-shot, we intentionally restrict its candidate pool: in each round, the LLM proposes exactly 10 candidate conditions. The local reasoning function $R_i(\theta)$ then evaluates these 10 proposals, selecting the $\operatorname{argmax}$ to query in the true objective space. In contrast, the standard ADKO variant (without LLM integration) and other baselines optimize their acquisition function over the entire unobserved design space. Full system prompts and context construction rules are detailed in Section~\ref{sec:additional_results_sci_dis}.
We measure the optimization efficiency by tracking the number of experimental rounds required for the decentralized network to identify a condition within the top-3 highest known yields in the dataset. The maximum experimental budget is capped at $T=200$ rounds per agent, allowing us to evaluate both the initial acceleration provided by the LLM warm-start and the asymptotic convergence of the tokenized communication scheme.

\section{Additional Results}\label{sec:additiona_results}
\subsection{Neural Architectural Search}\label{sec:additional_results_nas}

\begin{table}[htbp]
\small
    \centering
    \begin{tabular}{lcc}
        \toprule
        Topology & ADKO (best-found mean test acc., \%) & Independent BO ($\Delta$ pp) \\
        \midrule
        Ring (2-regular) & $72.07$ & $+4.61$ \\
        Random geometric ($r{=}0.6$) & $72.07$ & $+4.61$ \\
        Fully connected & $72.59$ & $+5.13$ \\
        \bottomrule
    \end{tabular}
    \caption{Best-found mean test accuracy (\%) under different communication topologies, IID, $N=5$ agent, for NAS experiment.}
    \label{tab:nas_topology}
\end{table}

\textbf{Communication Topology for IID.}
Table~\ref{tab:nas_topology} indicates that topology matters, but only to a limited extent in this IID NAS setting. Performance is best under full connectivity and slightly lower under the random geometric and ring graphs, which is qualitatively consistent with Theorem~\ref{thm:main}: better-connected graphs have more favorable mixing behavior, reducing the topology-dependent penalty in the regret bound. Through the link between $\lambda_2(L(\mathcal{G}))$ and the spectral gap of the mixing matrix discussed in Section~\ref{sec:connection_fiedler_spectral_gap}, denser graphs should transmit useful tokens more rapidly and with less multi-hop delay. That trend is visible in the table, but the effect size is modest. In practice, ADKO remains competitive even under sparse communication, suggesting that its gains do not rely on exceptionally dense networks, though improved connectivity still provides a small and theory-consistent benefit.

\begin{figure}[htbp]
     \centering
     \begin{subfigure}[b]{0.48\textwidth}
         \centering
         \includegraphics[width=\textwidth]{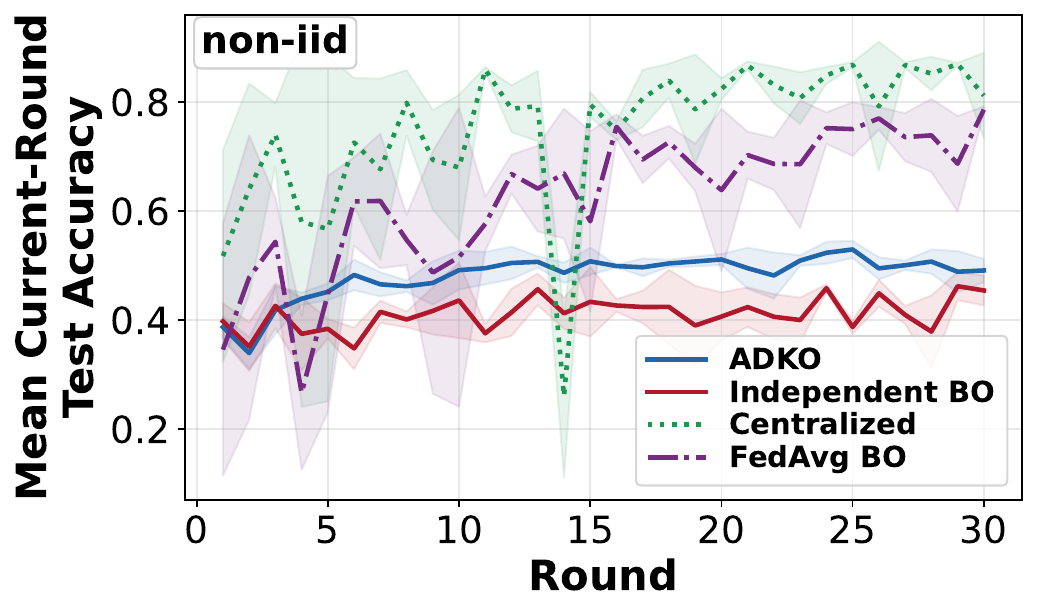}
         \caption{\textbf{Current-round mean test accuracy} comparison between ADKO and other frameworks.}
         \label{fig:nas_noniid_results}
     \end{subfigure}
     \hfill
     \begin{subfigure}[b]{0.48\textwidth}
         \centering
         \includegraphics[width=\textwidth]{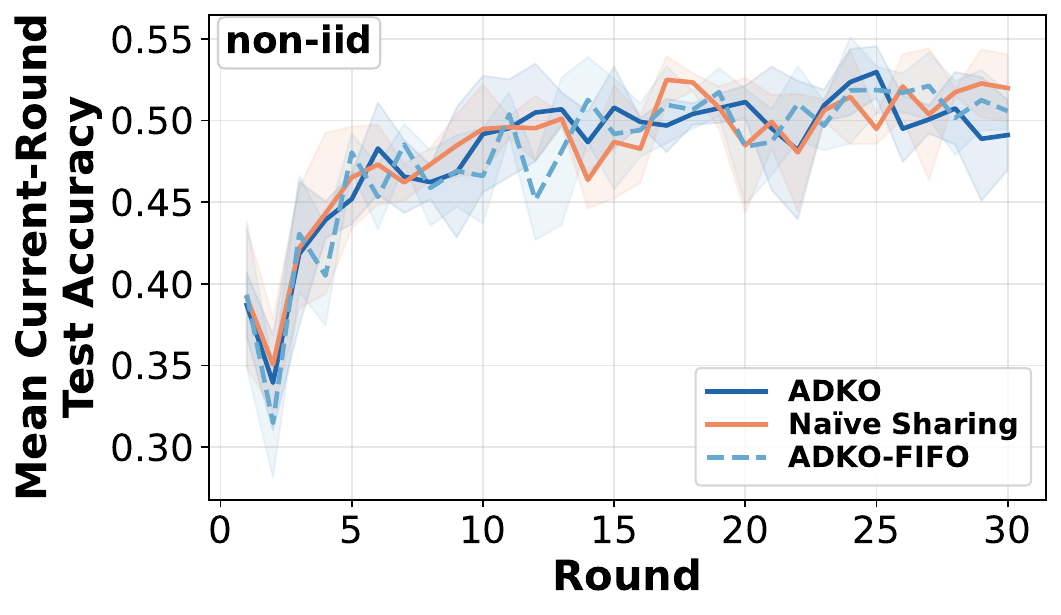}
         \caption{\textbf{Current-round mean test accuracy} comparison between ADKO and its variants.}
         \label{fig:nas_variants_noniid}
     \end{subfigure}
        \vspace{0.5ex}
     \caption{Results of Neural Architecture Search on CIFAR-10 with 80\% non-IID split.}
\end{figure}

\textbf{Non-IID Results.} In the non-IID setting, ADKO remains slightly above Independent BO over most of the horizon, but the gap is modest, as seen in Figure~\ref{fig:nas_noniid_results}, indicating that heterogeneity substantially reduces the value of shared information. Moreover, ADKO, Naive Sharing, and ADKO-FIFO exhibit overlapping uncertainty bands and no statistically clear separation (Figure ~\ref{fig:nas_variants_noniid}. Similarly, different topologies show no statsical difference, as seen in Table~\ref{tab:cifar10_topology_non_iid_best_found_mean}. One plausible explanation is that, in heterogeneous settings, the benefit of communication is highly sensitive to hyperparameter calibration. In particular, the relative weighting of communicated signals through introduced terms such as $\lambda$, $\gamma$, and $\tau$ can strongly affect whether peer information is amplified, suppressed, or treated as effectively neutral, and these choices need to be tuned differently for each setup. Thus, while ADKO shows a small advantage over fully independent optimization, the current results suggest that the strength of communication under non-IID conditions is at least partly calibration-limited. A more systematic study of communication-sensitive hyperparameter calibration is left for future work.

\begin{table}[htpb]
\small
    \centering
        \begin{tabular}{lcc}
        \toprule
        Topology & ADKO (best-found mean test acc., \%) & Independent BO ($\Delta$ pp) \\
        \midrule
        Ring (2-regular) & $53.31$ & $+3.50$ \\
        Random geometric ($r{=}0.6$) & $53.22$ & $+3.42$ \\
        Fully connected & $53.51$ & $+3.71$ \\
        \bottomrule
    \end{tabular}
    \caption{Best-found mean test accuracy (\%) under different communication topologies, non-IID, $N=5$ agents, for NAS experiment.}
\label{tab:cifar10_topology_non_iid_best_found_mean}

\end{table}

\subsection{Scientific Discovery}\label{sec:additional_results_sci_dis}

\label{sec:additional_iid_results}

\begin{figure}[htbp]
    \centering
    \includegraphics[width=0.6\linewidth]{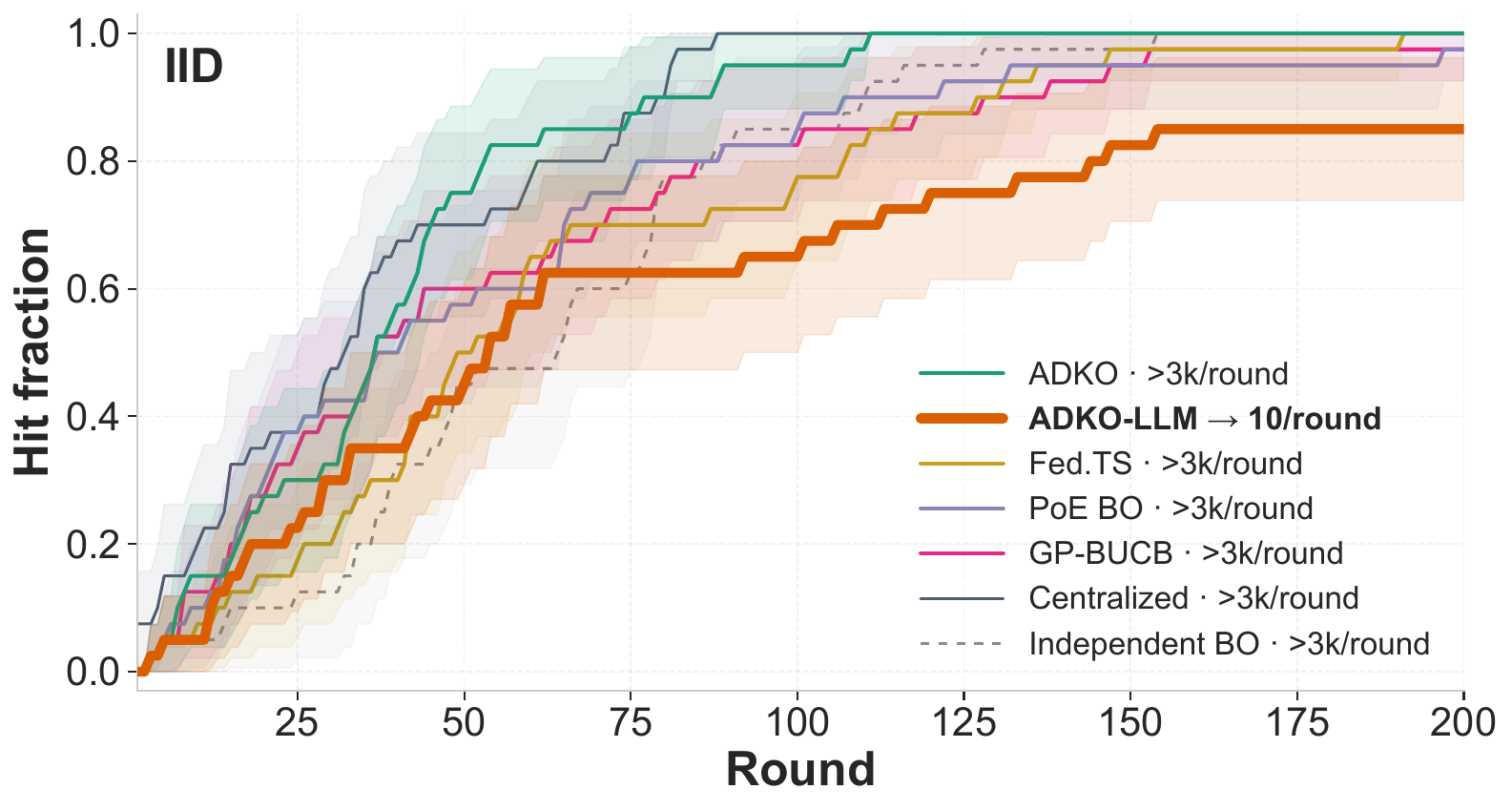}
    \caption{IID results on \texttt{suzuki\_edbo}. ADKO outperforms Different from the non-IID setting, round except ADKO+LLM, which evaluates 10 LLM-proposed candidates per round.}
    \label{fig:sci_dis_iid}
\end{figure}

We run the same experiment from Section~\ref{sec:additional_sci_dis_setup} on an IID regime, in which all four agents share the same search space (no solvent restriction). In this regime we additionally include \textit{Parallel GP-UCB}~\cite{desautels2014parallelizing} as a baseline: agents broadcast their intended evaluation candidates and temporarily condition their private GPs on hallucinated outcomes via the constant-liar heuristic \cite{desautels2014parallelizing}. This represents the standard ``intent-sharing'' paradigm to discourage redundant parallel evaluations while preserving exact outcome privacy. We omit Parallel GP-UCB from the non-IID main results because, when agents explore disjoint categorical spaces, a peer's hallucinated pending point provides no useful exploration guidance, effectively reducing the method back to Independent BO.

\textbf{Findings.} The IID results contain one clear positive for ADKO and one clear negative for the LLM variant.
ADKO outperforms every decentralized baseline (Fed.TS, PoE GP-UCB, Parallel GP-UCB, Independent GP-UCB) and tracks Centralized GP-UCB closely throughout the trajectory. Compact tokens carry sufficient signal in the IID regime to recover most of the collaborative benefit available to higher-bandwidth methods.
ADKO+LLM degrades in IID. ADKO+LLM is the weakest method shown, plateauing well below ADKO and the other decentralized baselines. We see two plausible mechanisms. First, because in IID all agents share the same search space, LLM proposals can collide across agents and lock the network into a narrow region. Second, an LLM prior tuned for non-IID chemistry is miscalibrated for IID, and a different architecture/prompting would be needed to lift the performance of ADKO-LLM.

We view this as a confirming observation: it is exactly the regime our theory anticipates. The LLM bias term in Theorem~\ref{thm:main} is linear in $T$ whenever the LLM's total variation distance $\TV$ from the optimum-conditional distribution is bounded away from zero, with Corollary~\ref{cor:sublinear} requiring $\TV = o(1)$ for sublinear regret and Proposition~\ref{prop:lower} showing this linear penalty is unavoidable in the worst case. The IID result is an empirical instance: when the four-term acquisition is already well-calibrated and there are no semantic gaps for the LLM to fill, a coarse LLM prior contributes positive $\TV$ without compensating value, and the bias term dominates.

\begin{figure}[htbp]
    \centering
    \includegraphics[width=0.7\linewidth]{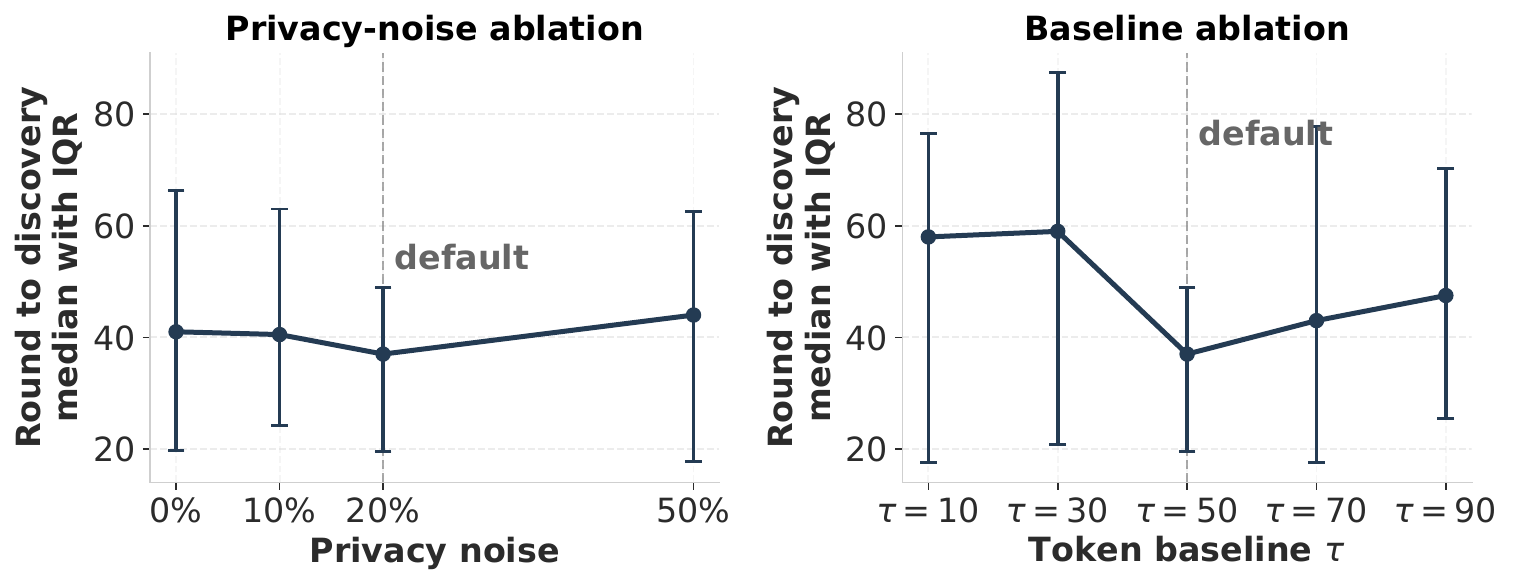}
    \caption{Noise and contextual-baseline sweep on the \texttt{suzuki\_edbo} task, IID setting. Lower is better.}
    \label{fig:sci_dis_ablations_iid}
\end{figure}
\textbf{Ablations.} Figure~\ref{fig:sci_dis_ablations_iid} confirms the results from the non-IID experiment. Some added noise can act as a regularization mechanism and incentivize exploration, and $\tau=50$ remains the strongest baseline for top-3 discovery in this problem.

\begin{figure}[htbp]
    \centering
    \begin{minipage}[t]{0.48\linewidth}
        \centering
        \includegraphics[width=\linewidth]{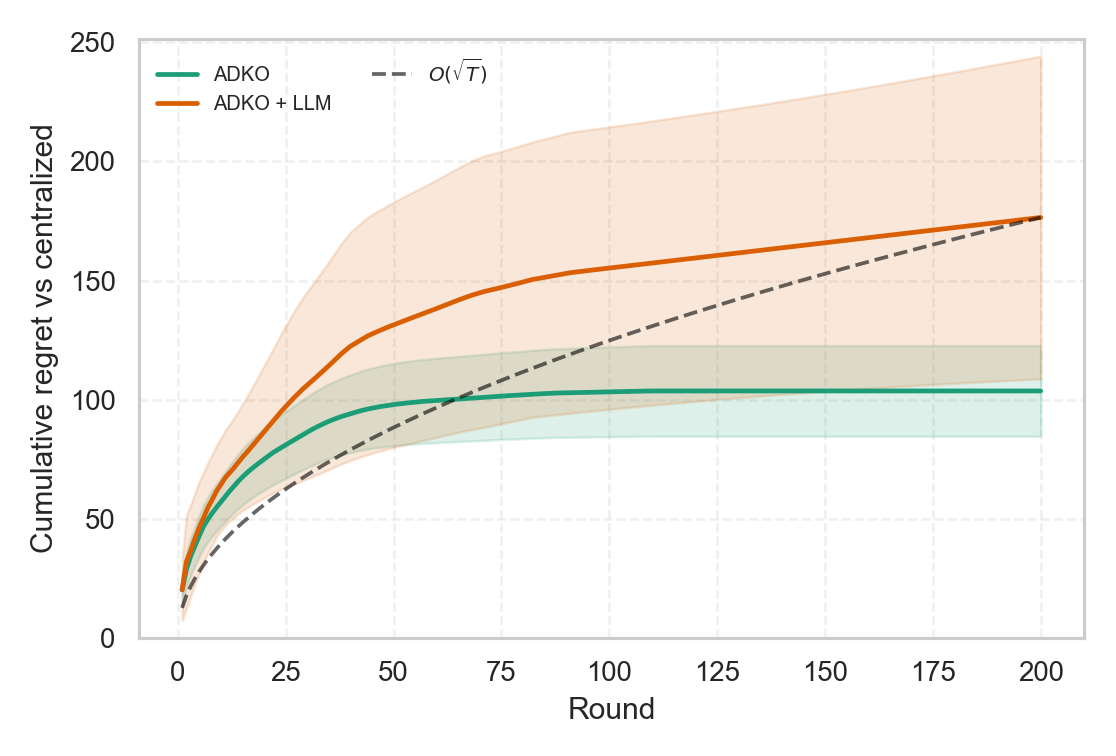
        }
        
        \vspace{0.5ex}
        {\small (a) IID}
    \end{minipage}
    \hfill
    \begin{minipage}[t]{0.48\linewidth}
        \centering
        \includegraphics[width=\linewidth]{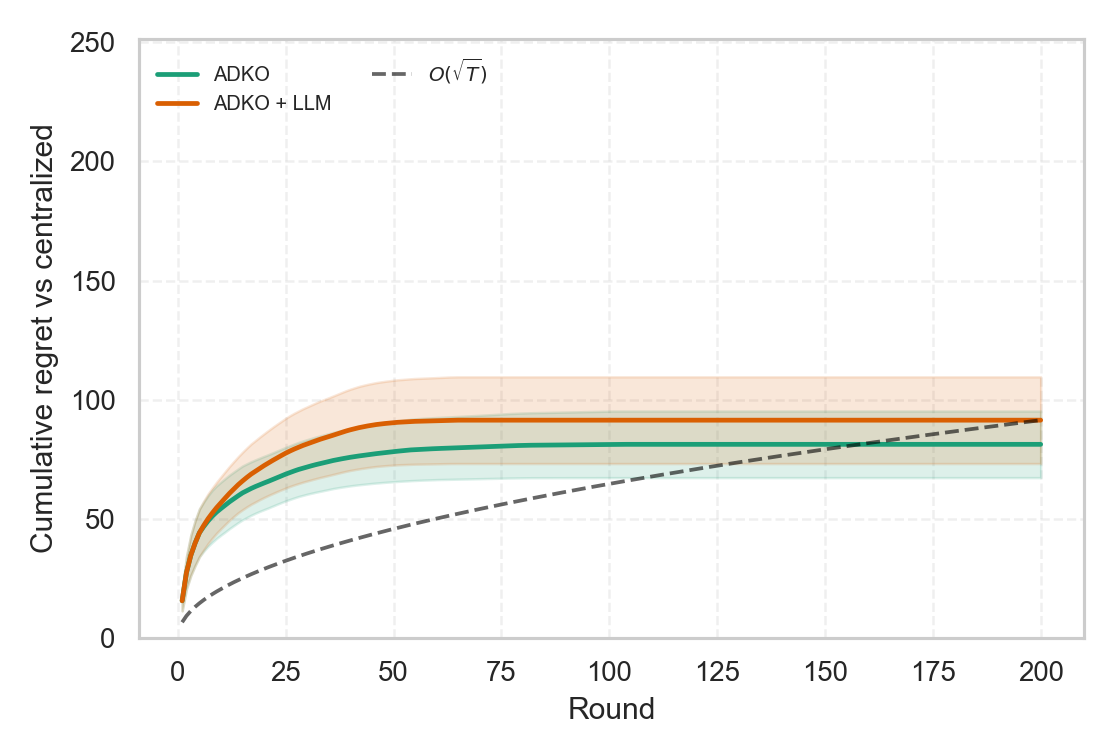}
        
        \vspace{0.5ex}
        {\small (b) Non-IID}
    \end{minipage}
    \caption{Cumulative regret for the scientific discovery experiment relative to a centralized baseline. The dashed $O(\sqrt{T})$ envelope is scaled to the maximum terminal regret strictly for visual shape comparison. Shaded regions denote 95\% confidence intervals. The early plateau in the non-IID regime illustrates that specialized agents match the centralized trajectory significantly faster than uniform IID agents.}
    \label{fig:regret}
\end{figure}

\textbf{Regret.}  Across both data regimes, cumulative regrets (measured relative to the centralized baseline) in Figure~\ref{fig:regret} exhibit sublinear growth (implied by Theorem~\ref{thm:main}), though the saturation dynamics differ significantly. In the IID setting, regret accumulates steadily over a longer horizon, with curves closely tracking the shape of the $O(\sqrt{T})$ reference. In contrast, the non-IID regime demonstrates a strong early saturation, reaching a near-zero instantaneous regret (a flat slope) roughly by mid-horizon. This early plateau indicates that non-IID agents swiftly match the centralized trajectory, ultimately converging faster and to a substantially lower terminal cumulative regret than their uniform IID counterparts. The oracle gaps in both scenarios also validate Proposition~\ref{prop:oracle}.

\textbf{Communication Cost.} The communication overhead of decentralized planners varies as problem complexity (such as the number of agents ($N$), model parameters ($m$), and candidate space size ($M$)) grows. Among the baselines, GP-BUCB requires the least bandwidth (under 1~KB per round for $N=4$), as it only shares intended queries, scaling at $\mathcal{O}(N^2)$ independent of the environment size. Server-FTS exchanges $m$-dimensional parameter vectors, yielding a moderate payload (approx.~32~KB) that scales at $\mathcal{O}(N^2 m)$ alongside model complexity. PoE-UCB is the most expensive baseline (approx.~710~KB), as agents must broadcast predictive posterior statistics across the entire candidate set, scaling at $\mathcal{O}(N^2 M)$---a prohibitive cost for massive search spaces. In contrast, our proposed ADKO method communicates only a fixed-size local summary (a state token, success bit, confidence, and fidelity), resulting in a highly efficient payload of roughly 0.7~KB per round. Even when augmented with textual reasoning in ADKO-LLM, the total size remains exceptionally lightweight at under 4~KB. Crucially, because the payloads for ADKO and ADKO-LLM are independent of both model parameters and the candidate space, their communication cost remains strictly bounded at $\mathcal{O}(N^2)$ overall, or $\mathcal{O}(1)$ with respect to the optimization problem size. Within our specific experimental setup, ADKO achieves this performance without a communication penalty: it empirically outperforms GP-BUCB while maintaining a similarly minimal footprint, and it surpasses our adapted Server-FTS baseline while requiring a fraction of the bandwidth. While this does not imply a generalized superiority over Federated Thompson Sampling across all domains, these results demonstrate that ADKO provides a highly competitive, low-overhead coordination strategy for similarly constrained environments.

\textbf{Transfer Knowledge Example.}
\label{app:llm:chains}
To illustrate cross-solvent transfer, we trace one chain (seed 52) in which a single recipe propagates across three labs, \textbf{DMF $\to$ THF $\to$ MeOH}, saturating at 100\% yield by round 9.
\begin{enumerate}
    \item  At round 6, the DMF agent proposes an experiment that achieves 93.79\% yield, with the rationale: \emph{``One-step from peer iodide/Bpin/SPhos success: keep reactive iodide and bulky biaryl ligand, but test alkoxide in DMF.''} It then broadcasts: \emph{``Heteroaryl iodides couple efficiently in DMF with bulky electron-rich biaryl phosphines like SPhos and alkoxide base, whereas less reactive chlorides or weaker ligands gave poor conversion.''}
    \item At round 7, the THF agent proposes an experiment that achieves 95.34\% yield, with the rationale: \emph{``Mirrors peer iodide/Bpin/lithium isopropoxide/SPhos success; keeps reactive C--I bond and bulky biaryl phosphine.''} It then broadcasts: \emph{``Heteroaryl iodides couple efficiently in THF with aryl Bpin esters when paired with a bulky electron-rich biaryl phosphine and alkoxide base, enabling fast oxidative addition and transmetalation.''}
    \item At round 9, the MeOH agent proposes an experiment that achieves 100.00\% yield, with the rationale: \emph{``Mirrors peer iodide+Bpin+LiOiPr+SPhos success; MeOH may still support this bulky biaryl phosphine/alkoxide activation.''}
\end{enumerate}

\section{LLM module: architecture, prompts, and behavior} \label{app:llm} 

This appendix specifies the auxiliary LLM module used by ADKO-LLM and analyzes its observed behavior. The module is intentionally narrow: it never produces a numeric score, never sees other agents' option sets, and never replaces the four-term acquisition $\mu + \beta\sigma + \lambda G - \gamma\Lambda$ \eqref{eq:reasoning}. Its only effects on selection are (i) shaping the candidate \emph{pool} the score acts on and (ii) attaching a one-sentence semantic payload $\ell$ to each shared token, which other agents' propose calls can read. The module exposes two well-typed entry points:
\begin{itemize}[leftmargin=*]
    \item \textsc{Encode}: given a freshly observed $(\theta_i^t, y_i^t, s_i^t, c_i^t)$, produce $\ell_i^t$ — a single sentence summarizing the mechanistic reading of that result.
    \item \textsc{Propose}: given agent $i$'s allowed option set, its local history, and a fidelity-ranked block of peer tokens, return up to $K=10$ candidate tuples each tagged with an \emph{intent} and a one-line rationale.                                   
  \end{itemize}  
  Both calls go through a pinned model snapshot at temperature $0$ with a fixed seed and a SQLite prompt-to-response cache, so any rerun with identical inputs reproduces the exact selection trajectory.               \subsection{Encode pathway}
  \label{app:llm:encode}
  The encode prompt receives the agent's profile, its allowed option set, a formatted recent history (with success/advantage flags and reagent class labels), and the new $(x, y)$. The system prompt fixes the chemistry vocabulary and constrains the output to strict JSON. Three regimes are distinguished by the advantage $c = |y-\tau|/y_{\text{scale}} \in [0,1]$: \begin{itemize}[leftmargin=*]
  \item $c \geq 0.3$ and $s=1$: a \textsc{success} sentence citing reagent classes (e.g.\ ``aryl iodide + bulky electron-rich phosphine in MeOH  benefits from fluoride activation and supports oxidative addition'').  \item $c \geq 0.3$ and $s=0$: a \textsc{failure} sentence citing plausible mechanistic causes (e.g.\ ``hydroxide in MeCN likely promoted protodeboronation'').
  \item $c < 0.3$: a \textsc{neutral} flag that explicitly marks the result as near-baseline and uninformative. This regime exists to prevent the LLM from confabulating a mechanism around statistical noise.     
  \end{itemize}
  The output $\ell_i^t$ is the privacy-respecting semantic payload attached to the redacted token; the token itself only carries $(\theta_{\text{noisy}}, s, c, \eta, i, t, \ell)$. Neither $y$ nor $\theta_{\text{true}}$ ever leaves the originating agent.
  
  \subsection{Propose pathway}
  \label{app:llm:propose}
  
  \textbf{Prompt structure.} The propose prompt has six sections: (i) lab profile (with the solvent restriction made explicit in HET), (ii) the agent's allowed option set (with class-level reagent context), (iii) compact search progress (observations, best $y$, rounds-since-improve, recent improvement), (iv) a coverage map listing under-explored option indices and unseen $(\text{electrophile, ligand})$ and $(\text{ligand, solvent})$ pairs, (v) local history rendered as \textit{recent} + \textit{best successes} + \textit{high-advantage failures} + \textit{near-baseline} exemplars, and (vi) the top $|\mathcal{P}|=20$ peer tokens ranked by the same fidelity score used in Algorithm~\ref{alg:pruning}, $\eta\, c\, e^{-\alpha_\tau \Delta t}$, with their $\ell$-text payloads inlined. The user prompt asks for two-to-four mechanism hypotheses and then a list of candidates of the form
  \[
    \bigl\{\,\mathrm{idx}_1, \ldots, \mathrm{idx}_5,\ \mathrm{intent}\in\mathcal{I},\
    \mathrm{hypothesis\_id},\ \mathrm{rationale}\,\bigr\}.
    \]

\textbf{Intent vocabulary.} The intent set is
 \[
 \mathcal{I} = \{\,\texttt{exploit\_success},\,\texttt{near\_success\_variant},\,\texttt{avoid\_failure},\,\texttt{diversity\_probe}\,\}
\]
with the prompt-level definitions
\begin{itemize}[leftmargin=*]
    \item \texttt{exploit\_success}: highly similar to a strong local or peer success.
    \item \texttt{near\_success\_variant}: a 1--2-component edit around a promising tuple.
    \item \texttt{avoid\_failure}: a deliberate move away from high-advantage failure motifs.
    \item \texttt{diversity\_probe}: a high-Hamming-distance probe into an under-covered region.
\end{itemize}

\textbf{Portfolio mix (proposal-time).} Rather than let the LLM pick its own intent distribution, the prompt prescribes counts per intent, drawn from a phase schedule keyed to the agent's observation count $n_{\text{obs}}$ and its rounds-since-improvement $\rho$:
\begin{center}
\begin{tabular}{l c c c c}
\toprule
Phase & exploit & near-variant & avoid-fail & diversity\\
\midrule                                                                                             
    $n_{\text{obs}} < 12$              & 0.15 & 0.15 & 0.20 & 0.50 \\
    $12 \le n_{\text{obs}} < 40$       & 0.20 & 0.25 & 0.20 & 0.35 \\                                    
    $n_{\text{obs}} \ge 40$            & 0.30 & 0.25 & 0.20 & 0.25 \\                                    
    \midrule                                                                                             
    Stagnation pulse $\rho \ge 8$  & $-$0.05 & $-$0.05 & $\cdot$ & $+$0.10 \\                            
    Stagnation pulse $\rho \ge 16$ & $-$0.10 & $-$0.05 & $+$0.05 & $+$0.10 \\                            
    \bottomrule                                                                                          
  \end{tabular}                                                                                          
  \end{center}                                             After applying the stagnation deltas the ratios are floored at $1/20$ and renormalized; the resulting integer counts (summing to $K=10$) are written verbatim into the prompt (``\textit{use this exact target mix: 4 exploit\_success, 5 near\_success\_variant, 4 avoid\_failure, 7 diversity\_probe}''). The LLM therefore has no degrees of freedom over the intent marginal — its job is to choose \emph{which} tuple goes under each label.
  
  \textbf{Within-batch structural constraints.} After parsing the JSON response, candidates are filtered, in order:
  \begin{enumerate}[leftmargin=*,label=(R\arabic*)]
  \item Duplicate tuples are dropped.
  \item Any \texttt{diversity\_probe} pick must differ in $\geq 3$ positions from every previously accepted pick in the same batch.
  \item Any pair of accepted picks must differ in $\geq 2$ positions, except when both are tagged \texttt{exploit\_success} or \texttt{near\_success\_variant}; these two intents are jointly licensed to cluster. Once a high-$\mu$ peak is identified, GP-UCB benefits from a tight cluster of local refinements rather than forced spread across all exploit-like proposals.
  \end{enumerate}
  These rules force the per-agent batch to span the space at multiple scales while still permitting tight refinement near a recognized peak.
  
  \textbf{Feasibility validation.} Indices outside the agent's allowed set are snapped to the nearest allowed value when a singleton remains; otherwise the candidate is dropped. In non-IID this strictly enforces the solvent lock $x_4 = i$ for agent $i$.
  
  \subsection{Integration with the four-term acquisition}  \label{app:llm:integration}
  With $K_{\text{tot}}$ set to 10 (the value used throughout this paper), the per-agent candidate pool $\mathcal{C}_i^t$ at round $t$ is built as: 
  \begin{enumerate}[leftmargin=*]
  \item Sanitize the LLM batch: drop tuples that are infeasible, already observed by agent $i$, or already in the pool.
  \item Take the surviving LLM picks, up to $K_{\text{LLM}}\!=\!\min(10,|\text{batch}|)$.
  \item Uniformly random-fill from $\mathcal{X}_i\setminus \mathcal{C}_i^t$  until $|\mathcal{C}_i^t| = K_{\text{tot}}$.
  \end{enumerate}
  The pool is then scored by $\mu + \beta\sigma + \lambda G - \gamma\Lambda$ exactly as in non-LLM ADKO; the four-term acquisition is unchanged. The LLM's action is therefore purely \emph{compositional}: it can include points the score would not have surfaced (although in this experiment non-LLM ADKO covers the full unobserved space) and it can omit points the score would have ranked highly. The $\ell$-text payload then propagates each agent's qualitative reading to peers, so the next round's propose calls condition on neighbors' insights, not just their numeric tokens. This is the only channel through which an LLM call at one agent influences selection at another.

  \begin{table}[htpb]
  \small
  \centering
  \begin{tabular}{l r r c}
    \toprule
    \textbf{Intent} & \textbf{Total LLM Proposals} & \textbf{Total LLM Selections} & \textbf{Median Yield [Q1, Q3]} \\
    \midrule
    Exploit Success      & 56,676 (20.0\%) & 3,320 (20.5\%) & 90.07 [81.74, 93.79] \\
    Near Success Variant & 28,338 (10.0\%) & 5,243 (32.4\%) & 85.21 [73.77, 92.03] \\
    Avoid Failure        & 56,677 (20.0\%) & 2,493 (15.4\%) & 75.94 [53.54, 88.59] \\
    Diversity Probe      & 141,690 (50.0\%) & 5,150 (31.8\%) & 72.69 [49.08, 87.49] \\
    \bottomrule
  \end{tabular}
  \caption{Summary of proposed and selected intents, along with the median and interquartile range [Q1, Q3] of their scores, across all seeds of ADKO-LLM from the main experiment.}
  \label{tab:intent_stats}
\end{table}
  
  \subsection{Behavior analysis: proposal vs.\ selection across intents}                                 
  \label{app:llm:table} 
  Table~\ref{tab:intent_stats} aggregates LLM proposal and selection counts across the 40-seed NON-IID-ADKO-LLM sweep, together with the median and interquartile range of the realized yield $y$ for the candidates the score function actually selected from each intent class. Three patterns are worth flagging.
  
  \textbf{Intent labels carry yield signal.}
  Selected medians are monotone in the intent's nominal exploitation level:
  $\texttt{exploit\_success}\,(90.07) \succ
  \texttt{near\_success\_variant}\,(85.21) \succ
  \texttt{avoid\_failure}\,(75.94) \succ
  \texttt{diversity\_probe}\,(72.69)$.
  The score function never sees the intent label — it ranks candidates only through $\mu, \sigma, G, \Lambda$ — yet the points it ends up selecting from each class have realized yields consistent with the LLM's own taxonomy. The labels are not vacuous post-hoc tags; they are predictive of the underlying chemistry the score then confirms.
  
  \textbf{Selection rates differ from proposal rates by a large factor.} The portfolio mix is enforced at proposal time, not at selection time. The score function is free to upweight high-$\mu$ exploit picks when the GP supports them and to ignore diversity probes when no candidate in the pool beats them on $\mu + \beta\sigma + \lambda G - \gamma\Lambda$. Empirically it does so: \texttt{exploit\_success} and \texttt{near\_success\_variant}  together contribute 30\% of proposals but 52.9\% of selections, while \texttt{diversity\_probe} drops from 50\% of proposals to 31.8\% of selections. This is exactly the asymmetry the design intends: the LLM provides \emph{breadth} per round; the four-term score chooses \emph{depth}.
  
  \textbf{Avoid-failure has the lowest selection rate despite a non-trivial median.} Although \texttt{avoid\_failure} candidates have a higher median yield than \texttt{diversity\_probe}, they are selected less often (15.4\% vs.\ 31.8\% of selections). This is consistent with the score function's structure: \texttt{avoid\_failure} picks are by construction far from peer-attractive points (so $\lambda G$ is small for them) and they are not high-$\mu$ exploits either, which leaves $\beta\sigma$ as the only term that could carry them — and in the non-IID pool, $\beta\sigma$ is typically dominated by \texttt{diversity\_probe}-tagged points with larger distance from observed history. \texttt{Avoid\_failure}'s value is therefore indirect: it shapes the LLM's own reasoning and the $\ell$-text channel on subsequent rounds, rather than driving immediate selection.
  
  \textbf{Coupling to the main result.} Table~\ref{tab:intent_stats} supports the interpretation that the LLM acts as a \emph{semantic prior} over the candidate pool: it admits cross-solvent transfers and one-edit refinements that the four privacy-preserving statistics cannot prefer in advance, while pruning candidates that random sampling would have introduced without chemical justification. The four-term score then chooses the actual pick over the resulting LLM-shaped pool, with the proposal-vs-selection asymmetry above as the empirical signature.

\end{document}